\documentclass{article}

\usepackage[accepted]{icml2024}

\usepackage{xcolor}
\usepackage{amsfonts}
\usepackage{amssymb}
\usepackage{amsmath}
\usepackage{amsthm}
\usepackage{subfigure}
\usepackage{graphicx}
\usepackage{microtype}
\usepackage{hyperref}
\usepackage{color}[dvipsnames]
\usepackage{algorithm}
\usepackage{algorithmic}
\usepackage{pgfplots}
\usepackage{tikz}
\usetikzlibrary{angles,quotes}
\pgfplotsset{compat=1.16}
\definecolor{LightRed}{rgb}{1,.6,.6}
\definecolor{LightGreen}{rgb}{.6,1,.6}

\theoremstyle{plain}
\newtheorem{theorem}{Theorem}[section]
\newtheorem{proposition}[theorem]{Proposition}
\newtheorem{lemma}[theorem]{Lemma}
\newtheorem{corollary}[theorem]{Corollary}
\theoremstyle{definition}
\newtheorem{definition}[theorem]{Definition}
\newtheorem{assumption}[theorem]{Assumption}
\newtheorem{example}[theorem]{Example}
\theoremstyle{remark}

\usepackage[textsize=tiny]{todonotes}

\icmltitlerunning{Minimizing $f$-Divergences by Interpolating Velocity Fields}

\begin{document}
\makeatletter
\DeclareFontFamily{U}{tipa}{}
\DeclareFontShape{U}{tipa}{m}{n}{<->tipa10}{}
\newcommand{\arc@char}{{\usefont{U}{tipa}{m}{n}\symbol{62}}}%

\newcommand{\arc}[1]{\mathpalette\arc@arc{#1}}

\newcommand{\arc@arc}[2]{%
  \sbox0{$\m@th#1#2$}%
  \vbox{
    \hbox{\resizebox{\wd0}{\height}{\arc@char}}
    \nointerlineskip
    \box0
  }%
}
\makeatother
\newtheorem{mytheorem}{Theorem}
\newtheorem{mylemma}{Lemma}
\newcommand{\argmax}{\mathop{\rm argmax}\limits}
\newcommand{\argmin}{\mathop{\rm argmin}\limits}
\newcommand{\TpPsi}{{T_{p(\boldx;\boldtheta)}\boldpsi(\boldx)}}
\newcommand{\phatp}{\hat{p}(X_p)}
\newcommand{\arcrocstar}{\arc{\mathrm{ROC}}(t^*)}

\newcommand{\hilight}[1]{{#1}}
\newcommand{\highlight}[2][yellow]{\mathchoice%
  {\colorbox{#1}{$\displaystyle#2$}}%
  {\colorbox{#1}{$\textstyle#2$}}%
  {\colorbox{#1}{$\scriptstyle#2$}}%
  {\colorbox{#1}{$\scriptscriptstyle#2$}}}%
\newcommand{\infHessLL}{\lambda_{\min}\left[\nabla_\bolddelta^2\ell(\bolddelta)\right]}
\newcommand{\hessianLLstar}{\nabla_\bolddelta^2 \ell(\bolddelta^*)}
\newcommand{\hessianLL}{\nabla_\bolddelta^2 \ell(\bolddelta)}
\newcommand{\calX}{{\mathcal{X}}}
\newcommand{\calY}{{\mathcal{Y}}}
\newcommand{\boldtheta}{{\boldsymbol{\theta}}}
\newcommand{\bolddelta}{{\boldsymbol{\delta}}}
\newcommand{\boldthetaP}{{\boldsymbol{\theta}}^{(p)}}
\newcommand{\boldthetaQ}{{\boldsymbol{\theta}}^{(q)}}
\newcommand{\boldthetaPtop}{{\boldsymbol{\theta}}^{(p)\top}}
\newcommand{\factorp}{{\phi}^P}
\newcommand{\factorq}{{\phi}^Q}
\newcommand{\boldalpha}{{\boldsymbol{\alpha}}}
\newcommand{\boldHh}{{\widehat{\boldH}}}
\newcommand{\boldeta}{{\boldsymbol{\eta}}}
\newcommand{\bolda}{{\boldsymbol{a}}}
\newcommand{\boldH}{{\boldsymbol{H}}}
\newcommand{\boldA}{{\boldsymbol{A}}}
\newcommand{\boldB}{{\boldsymbol{B}}}
\newcommand{\boldC}{{\boldsymbol{C}}}
\newcommand{\boldS}{{\boldsymbol{S}}}
\newcommand{\boldK}{{\boldsymbol{K}}}
\newcommand{\boldmu}{{\boldsymbol{\mu}}}
\newcommand{\boldJ}{{\boldsymbol{J}}}
\newcommand{\boldT}{{\boldsymbol{T}}}
\newcommand{\boldTheta}{{\boldsymbol{\Theta}}}
\newcommand{\boldf}{{\boldsymbol{f}}}
\newcommand{\boldepsilon}{{\boldsymbol{\epsilon}}}
\newcommand{\boldu}{{\boldsymbol{u}}}
\newcommand{\boldm}{{\boldsymbol{m}}}
\newcommand{\bolds}{{\boldsymbol{s}}}
\newcommand{\KLpq}{\mathrm{KL}[p,q]}
\newcommand{\KLqp}{\mathrm{KL}[q,p]}
\newcommand{\ChiPpq}{\chi^2_\mathrm{p}[p,q]}
\newcommand{\ChiNpq}{\chi^2_\mathrm{n}[p,q]}
\newcommand{\boldone}{{\boldsymbol{1}}}
\newcommand{\boldxi}{{\boldsymbol{\xi}}}
\newcommand{\boldSigma}{{\boldsymbol{\Sigma}}}
\newcommand{\boldv}{{\boldsymbol{v}}}
\newcommand{\boldM}{{\boldsymbol{M}}}
\newcommand{\boldW}{{\boldsymbol{W}}}
\newcommand{\boldk}{{\boldsymbol{k}}}
\newcommand{\boldb}{{\boldsymbol{b}}}
\newcommand{\boldbeta}{{\boldsymbol{\beta}}}
\newcommand{\boldDelta}{{\boldsymbol{\Delta}}}
\newcommand{\nnu}{\nsample}
\newcommand{\nsample}{n}
\newcommand{\subsetr}{\boldsymbol{r}}
\newcommand{\boldthetah}{{\widehat{\boldtheta}}}
\newcommand{\mathbbR}{\mathbb{R}}
\newcommand{\KL}{\mathrm{KL}}
\newcommand{\numparams}{n}
\newcommand{\boldhh}{{\widehat{\boldh}}}
\newcommand{\boldh}{{\boldsymbol{h}}}
\newcommand{\Hh}{{\widehat{H}}}
\newcommand{\boldxnu}{\boldY}
\newcommand{\boldPhi}{{\boldsymbol{\Phi}}}
\newcommand{\boldvarphi}{{\boldsymbol{\varphi}}}
\newcommand{\boldx}{{\boldsymbol{x}}}
\newcommand{\bolde}{\boldsymbol{e}}
\newcommand{\boldxp}{{\boldsymbol{x}}_{p}}
\newcommand{\boldxq}{{\boldsymbol{x}}_{q}}
\newcommand{\boldz}{{\boldsymbol{z}}}
\newcommand{\boldg}{{\boldsymbol{g}}}
\newcommand{\boldV}{{\boldsymbol{V}}}
\newcommand{\boldw}{{\boldsymbol{w}}}
\newcommand{\boldr}{{\boldsymbol{r}}}
\newcommand{\boldQ}{{\boldsymbol{Q}}}
\newcommand{\boldF}{{\boldsymbol{F}}}
\newcommand{\boldphi}{{\boldsymbol{\phi}}}
\newcommand{\boldzero}{{\boldsymbol{0}}}
\newcommand{\thetahat}{{\hat{\boldsymbol{\theta}}}}
\newcommand{\thetaShat}{{\hat{\boldsymbol{\theta}}_S}}
\newcommand{\thetaSchat}{{\hat{\boldsymbol{\theta}}_{S^c}}}
\newcommand{\zhat}{{\hat{\boldsymbol{z}}}}
\newcommand{\zSchat}{{\hat{\boldsymbol{z}}_{S^c}}}
\newcommand{\zShat}{{\hat{\boldsymbol{z}}_{S}}}
\newcommand{\nde}{\nsample'}
\newcommand{\boldxde}{\boldY'}
\newcommand{\boldX}{{\boldsymbol{X}}}
\newcommand{\boldY}{{\boldsymbol{Y}}}
\newcommand{\boldy}{{\boldsymbol{y}}}
\newcommand{\boldt}{{\boldsymbol{t}}}
\newcommand{\boldYnu}{{\boldsymbol{Y}}}
\newcommand{\boldYde}{{\boldsymbol{Y}}}
\newcommand{\boldpsi}{{\boldsymbol{\psi}}}
\newcommand{\hh}{{\widehat{h}}}
\newcommand{\boldI}{{\boldsymbol{I}}}
\newcommand{\PE}{{\widehat{PE}}}
\newcommand{\ratioh}{\widehat{\ratiosymbol}}
\newcommand{\ratiosymbol}{r}
\newcommand{\ratiomodel}{g}
\newcommand{\thetah}{{\widehat{\theta}}}
\newcommand{\mathbbE}{\mathbb{E}}
\newcommand{\pnu}{p_\mathrm{te}}
\newcommand{\pde}{p_\mathrm{rf}}
\newcommand{\refsection}{\boldS_\mathrm{rf}}
\newcommand{\tesection}{\boldS_\mathrm{te}}
\newcommand{\refY}{\boldY_\mathrm{rf}}
\newcommand{\teY}{\boldY_\mathrm{te}}
\newcommand{\nseg}{n}
\newcommand{\distP}{P}
\newcommand{\distQ}{Q}
\newcommand{\iid}{\stackrel{\mathrm{i.i.d.}}{\sim}}
\newcommand{\dx}{\mathrm{d}\boldx}
\newcommand{\dy}{\mathrm{d}\boldy}
\newcommand{\hatE}{\widehat{\mathbbE}}
\newcommand{\gxeta}{g(\boldx;\boldeta)}
\newcommand{\Zeta}{Z(\boldeta)}
\newcommand{\Zetahat}{\hat{Z}(\boldeta)}
\newcommand{\FPR}{\mathrm{FPR}}
\newcommand{\TPR}{\mathrm{TPR}}

\def\ratio{r}
\def\relratio{{\ratio}_{\alpha}}

\def\ci{\perp\!\!\!\perp} %
\newcommand\independent{\protect\mathpalette{\protect\independenT}{\perp}} %
\def\independenT#1#2{\mathrel{\rlap{$#1#2$}\mkern2mu{#1#2}}} 
\newcommand*\xor{\mathbin{\oplus}}

\newcommand{\vertiii}[1]{{\left\vert\kern-0.25ex\left\vert\kern-0.25ex\left\vert #1 
    \right\vert\kern-0.25ex\right\vert\kern-0.25ex\right\vert}}

\twocolumn[
\icmltitle{Minimizing $f$-Divergences by Interpolating Velocity Fields}

\icmlkeywords{Machine Learning, ICML}

\icmlsetsymbol{equal}{*}

\begin{icmlauthorlist}
\icmlauthor{Song Liu}{uob}
\icmlauthor{Jiahao Yu}{uob}
\icmlauthor{Jack Simons}{uob}
\icmlauthor{Mingxuan Yi}{uob}
\icmlauthor{Mark Beaumont}{uob}
\end{icmlauthorlist}

\icmlaffiliation{uob}{University of Bristol, Bristol, UK
}

\icmlcorrespondingauthor{Song Liu}{song.liu@bristol.ac.uk}

\vskip 0.3in
]
\printAffiliationsAndNotice{}
\begin{abstract}
Many machine learning problems can be seen as approximating a \textit{target} distribution using a \textit{particle} distribution by minimizing their statistical discrepancy. 
Wasserstein Gradient Flow can move particles along a path that minimizes the $f$-divergence between the target and particle distributions. To move particles, we need to calculate the corresponding velocity fields derived from a density ratio function between these two distributions. Previous works estimated such density ratio functions and then differentiated the estimated ratios. These approaches may suffer from overfitting, leading to a less accurate estimate of the velocity fields. Inspired by non-parametric curve fitting, 
we directly estimate these velocity fields using interpolation techniques. 
We prove that our estimators are consistent under mild conditions. 
We validate their effectiveness using novel applications on domain adaptation and missing data imputation. The code for reproducing our results can be found at \url{https://github.com/anewgithubname/gradest2}. This manuscript is an extended version of the ICML2024 version. 

\end{abstract}

\section{Introduction}

Many machine learning problems can be formulated as minimizing statistical divergences between distributions, e.g., Variational Inference \citep{Blei2017}, Generative Modeling \citep{Nowozin2016b, yi2023monoflow}, Domain Adaptation \citep{courty2017joint,yu2021divergence}, and Data Imputation \citep{muzellec2020missing}. 
Among many divergence minimization techniques, 
Particle-based gradient descent reduces the divergence between a set of particles (which define a distribution) and the target distribution by iteratively moving particles according to an update rule. 

One such algorithm is Stein Variational Gradient Descent (SVGD) \citep{LiuQ2016SVGD,Liu2017stein}. It minimizes the Kullback-Leibler (KL) divergence using the steepest descent algorithm and has achieved promising results in Bayesian inference. However, to compute the SVGD updates, we need unnormalized target density functions. If we only have samples from the target distribution—as is often the case in applications like domain adaptation and generative model training—    we cannot directly apply SVGD to minimize the $f$-divergences.  

Wasserstein Gradient Flow (WGF) describes the evolution of a marginal measure $q_t$ along the steepest descent direction of a functional objective in Wasserstein geometry where $x_t \sim q_t$ follows a probability flow ODE. 
It can be employed to minimize an $f$-divergence between a particle distribution and the target distribution. 
Particularly, such WGFs characterize the following ODE 
\citep{yi2023monoflow,gao2019deep}
\[\dx_t = \nabla (h\circ r_t)(\boldx_t)\mathrm{d}t, ~~~ t\in [0, \infty). \] 
Here $r_t := \frac{p}{q_t}$ is the ratio between the target density $p$ and particle density $q_t$ at time $t$, and $h$ is a known function which depends on the $f$-divergence. The gradient operator $\nabla$ computes the gradient of the composite function $h\circ r_t$. 
If we know $r_t$, simulating the above ODE would be straightforward using a discrete-time Euler method. 
However, the main challenge is that we do not know the ratio $r_t$ in practice. In the context we consider, we know neither the particle nor target density, which constitutes $r_t$.

\begin{figure}
    \centering
    \includegraphics[trim = {.7cm, 0cm, .7cm, 0cm}, width=.48\textwidth]{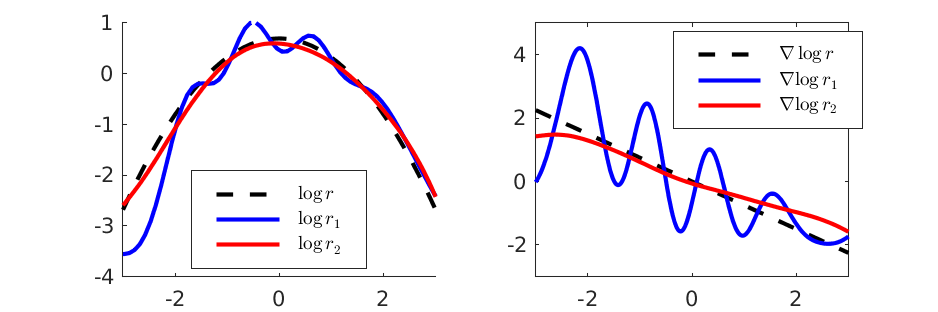}
    \caption{
    Estimating a log density ratio $ \log r$ using a flexible model (RBF kernel) leads to a overfitted estimate ({\color{blue}{$\log r_1$}}). 
    The overfitting consequently causes huge fluctuations in the derivative {\color{blue}$(\log r_1)'$}. Our proposed method provides a much more stable estimate {\color{red}$\log r_2$} and a more accurate estimate of {\color{red}$(\log r_2)'$}. 
    }\label{fig.illustration}
\end{figure}

To overcome this issue, recent works \citep{gao2019deep,ansarirefining2021,Simons2021Variational} first obtain an estimate $\hat{r}_t$ using a density ratio estimator \citep{Sugiyama2012}, then differentiate  $h\circ \hat{r}_t$ to obtain $\nabla (h\circ \hat{r}_t)(\boldx_t)$. However, like other estimation tasks, density ratio estimation can be prone to overfitting. The risk of overfitting is further exacerbated when employing flexible models such as kernel models or neural networks. 
Overfitting can be disastrous for gradient estimation: A wiggly fit of $r_t$ will cause huge fluctuations in gradient (see the blue fit in Figure \ref{fig.illustration}). 
Moreover, density ratio estimation lacks the inductive bias for the density ratio gradient estimation. 
In other words, while it might be good at estimating the ratio itself, it doesn't have any built-in assumptions to capture the gradient of the density ratio function accurately.

In this paper, we \textbf{directly approximate the velocity fields induced by WGF}, i.e., $\nabla (h\circ r_t)(\boldx_t)$.
We show that the backward KL velocity field, where $h = \log$, can be effectively estimated using Nadaraya-Watson (NW) interpolation if we know $\nabla \log p$, and this estimator is closely related to SVGD.  
We prove the estimation error of $\nabla \log  r_t(\boldx_t)$ vanishes as the kernel bandwidth approaches to zero.
This finding motivates us to propose a more general linear interpolation method to approximate $\nabla (h\circ r_t)(\boldx_t)$ for any general $h$ functions using only samples from $p$ and $q_t$.
Our estimators are based on the idea that, within the neighbourhood of a given point, the best linear approximation of $h \circ r_t$ has slope 
$\nabla (h \circ r_t)$.    
Under mild conditions, we show that our estimators are also consistent for estimating $\nabla (h\circ r_t)(\boldx_t)$ and achieve the optimal non-parametric regression rate. Finally, equipped with our proposed gradient estimator, we test WGF on two novel applications: \textit{domain adaptation} and \textit{missing data imputation} and achieve promising performance.

\section{Background}

\textbf{Notation}: $\mathbb{R}^d$ is the $d$-dimensional real domain. Vectors are lowercase bold letters, e.g., $\boldx := [x_1, \ldots, x_d]^\top$. $\boldx_{-j}$ is a subvector of $\boldx$ obtained by excluding the $j$-th dimension. Matrices are uppercase bold letters, e.g., $\boldX, \boldY$. $f \circ g$ is the composite function $f(g(\cdot))$. 
$f'$ is the derivative of a univariate function. $\partial_{i} f$ means the partial derivative with respect to the $i$-th input of $f$  and 
$\nabla f := [\partial_1 f, \partial_2 f, \cdots, \partial_d f]^\top$. 
$\nabla^\top f$ represents its transpose. 
$\nabla_\boldx f(\boldx,\boldy)$ represents the gradient of $f$ with respect to the input $\boldx$. 
$\nabla \boldf \in \mathbbR^{d \times m}$ represents the Jacobian of a vector-valued function $\boldf : \mathbbR^m \to \mathbbR^d$.
$\partial^2_{i} f$ means the second-order partial derivative with respect to the $i$-th input of $f$. 
$\nabla^2 f$ is the Hessian of $f$. 
$\|\cdot\|$ is the $\ell_2$ norm of a vector or the spectral norm of a matrix. 
$\lambda_{\min}\left[\boldX\right]$ represents the smallest eigenvalue of a matrix $\boldX$. 
$a \vee b$ is the greater value between two scalars $a$ and $b$.
$\mathcal{P}(\mathbb{R}^d)$ is the space of probability measures defined on $\mathbb{R}^d$ equipped with Wasserstein-2 metric or Wasserstein space for short. $\hat{\mathbbE}[\cdot]$ is the sample approximation of an expectation $\mathbbE[\cdot]$.

We begin by introducing WGF of $f$-divergence, an effective technique for minimizing $f$-divergence between the particle and target distribution. 

\subsection{Wasserstein Gradient Flows of $f$-divergence}

In general terms, a Wasserstein Gradient Flow is a curve in probability space \citep{ambrosio2005gradient}. By moving a probability measure along this curve, a functional objective (such as a statistical divergence) is reduced. 
In this work, we focus solely on using $f$-divergences as the functional objective.
Let $q_t: \mathbb{R}_+ \to \mathcal{P}(\mathbb{R}^d)$ be a curve in Wasserstein space. 
Consider an $f$-divergence defined by $D_f[p, q_t] := \int q_t(\boldx) f(r_t(\boldx)) \dx$, where $r_t:= \frac{p}{q_t}$. $f$ is a twice differentiable convex function with $f(1) = 0$.  
\begin{theorem}
[Corollary 3.3 in \citep{yi2023monoflow}]
\label{thm.gf}
The Wasserstein gradient flow of $D_f[p, q_t]$ characterizes the particle evolution via the ODE:
\begin{align*}
    \dx_t &= \nabla (h\circ r_t)(\boldx_t) \mathrm{d} t,~ h(r_t) = r_tf'(r_t) - f(r_t) .
\end{align*}
\end{theorem}
Simply speaking, particles evolve in Euclidean space according to the above ODE moves the corresponding $q_t$ along a curve where $D_f[p, q_t]$ always decreases with time. 

Theorem \ref{thm.gf} establishes a relationship between $f$ and a function $h :\mathbbR_+ \mapsto \mathbbR$. 
The gradient field of $h \circ r_t$ over time as $t \to \infty$ is referred to as the \textit{WGF velocity field}. 
Note that, in \citet{gao2019deep} and \citet{ansarirefining2021}, authors provided a similar theorem for the ``reversed'' $f$-divergence $\int p(\boldx) f(\frac{q_t(\boldx)}{p(\boldx)}) \mathrm{d}\boldx$, which is different from the definition of $f$-divergence that we use in our paper.
Some frequently used $f$-divergences and their corresponding $h$ functions are listed in Table \ref{table:f_div_ode}. 
Specifically, for the backward KL divergence we have $\nabla (h\circ r_t)(\boldx_t)\vert_{h = \log(\cdot) } = \nabla \log r_t(\boldx_t)$. 

\begin{table}[t]
\centering
\begin{tabular}{|c|c|c|c|}
\hline
\textbf{Name} & \textbf{Notation} & $f(r_t)$ & $h(r_t)$ \\ \hline
Forw. KL & $\mathrm{KL}[p,q_t]$ & $r_t\log(r_t)$ & $r_t$ \\ \hline
Back. KL & $\mathrm{KL}[q_t, p]$ & $-\log(r_t)$ & $\log r_t - 1$ \\ \hline
Pearson's $\chi^2$ & $\chi_\mathrm{p}^2[p,q_t]$ & $\frac{1}{2} (r_t - 1)^2$ & $\frac{1}{2}r_t^2 - \frac{1}{2}$ \\ \hline
Neyman's $\chi^2$ & $\chi_\mathrm{n}^2[p,q_t]$ & $\frac{1}{2r_t} - \frac{1}{2}$ & $-\frac{1}{r_t} + \frac{1}{2}$ \\
\hline
\end{tabular}
\caption{Some $f$-divergences, their definitions and $h$ functions computed according to Theorem \ref{thm.gf}.}
\label{table:f_div_ode}
\vspace*{-5mm}
\end{table}

In reality, we move particles by simulating the above ODE using the forward Euler method:
We draw particles from an initial distribution $q_0$ and iteratively update them for time $t = 0, 1 \dots T$ according to the following rule:
\begin{align}
\label{eq.wgf}
    \boldx_{t+1} := \boldx_{t} + \eta \nabla (h\circ r_t)(\boldx_t)
\end{align}
where $\eta$ is a small step size. There is a slight abuse of notation, and we reuse $t$ for discrete-time indices \footnote{From now on, we will only discuss discrete-time algorithms.}.

Although \eqref{eq.wgf} seems straightforward, we normally do not have access to $r_t$, so the update in \eqref{eq.wgf} cannot be readily performed. In previous works, such as \citep{gao2019deep,Simons2021Variational}, $r_t$ is estimated using density ratio estimators and the WGF is simulated using the estimated ratio. Although these estimators achieved promising results, the density ratio estimators are not designed for usage in WGF algorithms. For example, a small density ratio estimation error could lead to huge deviations in gradient estimation, as we demonstrated in Figure \ref{fig.illustration}. Others \citep{wang2022projected} propose to estimate the gradient flow using Kernel Density Estimation (KDE) on densities $p$ and $q_t$ separately (See Section \ref{sec.kde.app}), then compute the log ratio and its gradient. However, KDE tends to perform 
poorly in high dimensional settings (see e.g., \citep{scott1991feasibility}). 

\section{Direct Velocity Field Estimation by Interpolation}
In this work, we consider directly estimating the velocity field, i.e., directly modelling and estimating  $\nabla (h\circ r_t)$. We are encouraged by the recent successes in Score Matching \citep{Hyvaerinen2005,hyvarinen2007,vincent2011connection,song2020sliced},
which is a direct estimator of a log density gradient. It works by minimizing the squared differences between the true log density gradient and the model gradient.
However, such a technique cannot be easily adapted to estimate $\nabla (h\circ r_t)$, even for $h(r) = \log r$ (See Appendix \ref{sec.why.sm.doesnt.work}).

We start by looking at a simpler setting where $\nabla \log p$ is known. In fact, this setting itself has many interesting applications such as Bayesian inference. The solution we derive using interpolation will serve as a motivation for other interpolation based approaches in later sections.

\subsection{Nadaraya-Watson (NW) Interpolation of Backward KL Velocity Field}
\label{sec.nw.interpo}
Define a local weighting function with a parameter $\sigma >0$, 
$k_\sigma(\boldx, \boldx^\star) := \exp\left(- \frac{\|\boldx - \boldx^\star\|^2}{2\sigma^2}\right).$  
Nadaraya-Watson (NW) estimator \citep{Nadaraya1964,Watson1964} interpolates a function $g$ at \textbf{a fixed point} $\boldx^\star$. Suppose that we observe $g(\boldx)$ at a set of sample points $\{\boldx_i\}_{i=1}^n \sim q$, NW interpolates $g(\boldx^\star)$ by computing
\begin{align}
\label{eq.NW}
    \hat{g}(\boldx^\star) := {
    \widehat{\mathbbE}_q[k_\sigma(\boldx, \boldx^\star)g(\boldx)]}/{\widehat{\mathbbE}_q[k_\sigma(\boldx, \boldx^\star)]}. 
\end{align}
Thus, the NW interpolation of the backward KL field \footnote{``backward KL field'' is short for backward KL velocity field. } is 
\begin{align}
\label{eq.NW.1}
    {
    \hat{\boldu}_t(\boldx^\star) := \widehat{\mathbbE}_{q_t}[k_\sigma(\boldx, \boldx^\star)\nabla \log r_t(\boldx)]}/{\widehat{\mathbbE}_{q_t}[k(\boldx, \boldx^\star)]}.
\end{align}
Since we cannot evaluate $\nabla \log r_t(\boldx)$, \eqref{eq.NW.1} is intractable. 
However, assuming that 
$\lim_{\|\boldx\| \to \infty} q_t(\boldx)k(\boldx,\boldx^\star) = 0$, using integration by parts
\footnote{
\begin{align*}
    & \int \int q(\boldx) k(\boldx, \boldx^\star) \partial_i \log q(\boldx) \mathrm{d}x_i\dx_{-i}  \\
    = & \int  \left[ \left[q(\boldx)k(\boldx, \boldx^\star) \right]_{x_i \to -\infty}^{x_i \to \infty} - \int  q(\boldx) \partial_i k(\boldx, \boldx^\star) \mathrm{d}x_i \right] \mathrm{d}\boldx_{-i}
\end{align*}
}
, 
the expectation in the numerator of \eqref{eq.NW.1} can be rewritten as: 
\begin{align}
\label{eq.nw.2}
    \mathbbE_{q_t}\left[k^\star_\sigma\nabla \log r_t(\boldx)\right]=&\mathbbE_{q_t}\left[k^\star_\sigma\nabla \log {p(\boldx)}\right] - \mathbbE_{q_t}\left[{\nabla \log q_t(\boldx)}\right] \notag \\
     = & 
    \mathbbE_{q_t}[k^\star_\sigma\nabla \log p (\boldx) + \nabla k^\star_\sigma], 
\end{align}
where we shortened the kernel $k_\sigma(\boldx, \boldx^\star)$ as $k^\star_\sigma$.
Since we can evaluate $\nabla \log p$, 
\eqref{eq.nw.2} can be approximated using samples from the particle distribution $q_t$. Thus, 
the NW estimator of the backward KL field can be approximated by  
\begin{align}
\label{eq.NW.3}
    \hat{\boldu}_t(\boldx^\star) \approx {
    \widehat{\mathbbE}_{q_t}[k^\star_\sigma\nabla \log p (\boldx) + \nabla k^\star_\sigma]}/{\widehat{\mathbbE}_{q_t}[k^\star_\sigma]}. 
\end{align}

Interestingly, the numerator of \eqref{eq.NW.3} is exactly the particle update of the SVGD algorithm \citep{LiuQ2016SVGD} for an RKHS induced by a Gaussian kernel (See Appendix \ref{sec.app.svgd} for details on SVGD), and the equality \eqref{eq.nw.2} has been noticed by \citet{Chewi2020svgd}. 
Note that for different $\boldx^\star$, the denominator in \eqref{eq.NW.3} is different and thus cannot be combined into the overall learning rate of SVGD.  

\subsection{Effectiveness of NW Estimator}
For simplicity, 
we drop $t$ from $\boldu_t$, $r_t$, $\boldx_t$ and $q_t$ when our analysis holds the same for all $t$.

Although there have been theoretical justifications for the convergence analysis of WGF given the ground truth velocity fields such as Langevin dynamics  \citep{wibisono2018sampling}. 
Few theories have been dedicated to the estimation of velocity fields themselves. 
One of the contributions of this paper is that we study the statistical theory of the velocity field estimation through the lenses of non-parametric regression/curve approximation. 

Now, we prove the convergence rate of the NW estimator under the assumption that the second-order derivative of $\log r$ is well-behaved. Although $\hat{\boldu}(\boldx^\star)$ cannot be directly computed, assuming $\nabla \log p$, $k$ and $\nabla k$ are well-behaved, using concentration inequalities (such as Hoeffding's inequality \citep{Hoeffding1963}), the difference between $\hat{\boldu}(\boldx^\star)$ and its approximation \eqref{eq.NW.3} can be easily bounded. Thus, we focus on the classical NW estimator in \eqref{eq.NW.1}.

\begin{proposition}
    \label{svgd.bias}
    Suppose $ \sup_{\boldx \in \mathbbR^d } \|\nabla^2 \log r(\boldx)\| \le \kappa < \infty$. 
    Define $k(\boldy) := \exp\left[ - \|\boldy\|^2/2 \right]$. 
    Assume that there exist constants $C_k, K$ that are 
    independent of $\sigma$, such that 
    \begin{align}
    \label{eq.nw.condition} 
    \frac{ \int q(\sigma \boldy + \boldx^\star) k(\boldy) \|\boldy\| \mathrm{d}\boldy }{{\int q(\sigma \boldy + \boldx^\star) k(\boldy)} \mathrm{d}\boldy } &\le C_k,     
    \mathbbE_q\left[\frac{1}{\sigma^d}k_\sigma^\star\right] \ge K > 0 \\
    \mathrm{Var}_q\left[\frac{1}{\sigma^d}k_\sigma^\star \| \boldx - \boldx^\star\|\right] & = O\left(\frac{1}{\sigma^d}\right), \notag \\
    \mathrm{Var}_q\left[\frac{1}{\sigma^d}k_\sigma^\star\right] &= O\left(\frac{1}{\sigma^d}\right), 
    \text{   as } \sigma \to 0. 
    \label{eq.vars}
    \end{align}
    Then, with high probability, there exists a constant $K'$:
    $$\|\hat{\boldu}(\boldx^\star) - \nabla \log r(\boldx^\star)\| \le \sqrt{d} \kappa \left[\frac{K'}{\sqrt{ n \sigma^d}} +  \sigma  C_k \right]$$ holds for all $\sigma > 0$. 
\end{proposition}
We can also deduce that, with an optimal choice of $\sigma \sim n^{-1/(d+2)}$, the estimation error is $O_p(n^{-1/(d+2)})$.

See Appendix \ref{sec.svgd.proof} for the proof.
$q(\sigma \boldy + \boldx^\star)$ in the first inequality \eqref{eq.nw.condition} can be further expanded using Taylor expansion. Provided that the kernel $k$ is well-behaved, this condition becomes a regularity condition on $q(\boldx^\star)$ and its higher order moments. 
The second inequality in \eqref{eq.nw.condition} means that there should be enough mass around $\boldx^\star$ under the distribution $q$, which is a key assumption in classical nonparametric curve estimation (See, e.g., Chapter 20 in \citep{Wasserman2010}). 
\eqref{eq.vars} is required to ensure the empirical quantities in the NW estimator converge to their population counterparts 
in probability.

It can be seen that the estimation error is bounded by the sample approximation error $\frac{K'}{\sqrt{n\sigma^d}}$ and a bias depending on by $\sigma$. Interestingly, the bias term decreases at the rate  of $\sigma$, slower than the classical rate $\sigma^2$ for the non-parametric regression of a second-order differentiable function (Theorem 20.21 in 
 \citep{Wasserman2010}). This is expected as NW estimates the \textit{gradient} of $h \circ r$, not $h \circ r$. To achieve a faster $\sigma^2$ rate, one needs to assume conditions on $\nabla ^3 \log r$. This also highlights a slight downside of using NW to estimate the gradient. However, in the following section, we show that another interpolator achieves the superior rate when using the same type of assumption on $\nabla^2 (h \circ r)$.

\section{Velocity Field Interpolation from Samples}

Although we have seen that $\hat{\boldu}$ is an effective estimator of the backward KL velocity field, 
it can only be approximated when we can evaluate $\nabla \log p$. In some applications, such as domain adaptation or generative modelling, we only have samples from the target distribution, and $\nabla \log p$ is unavailable. 
Moreover, since the $f$-divergence family consists of a wide variety of divergences, we hope to provide \textit{a general computational framework} to estimate different velocity fields that minimize different $f$-divergences.

Nonetheless, the success of NW
motivates us to look for other interpolators to approximate $\nabla (h \circ r)(\boldx^\star)$. 

Another common interpolation technique is \emph{local linear regression} (See e.g., \citep{GasserMuller1979,Fan1992} or Chapter 6, \citep{Hastie2001}). It approximates an unknown function $g$ at $\boldx^\star$ by using a linear function: $\hat{g}(\boldx) := \langle \boldbeta(\boldx^\star), \boldx \rangle + \beta_0(\boldx^\star).$ $\boldbeta(\boldx^\star)$ and $\beta_0(\boldx^\star)$ are the minimizer 
of the following weighted least squares objective:  
\begin{align}
\label{eq.local.linear.obj}
    \min_{\boldbeta \in \mathbbR^d, \beta \in \mathbbR} \widehat{\mathbbE}_{q} \left[ k_\sigma^\star \left(g(\boldx) - \langle \boldbeta, \boldx \rangle - \beta_0\right)^2 \right].
\end{align}

A key insight is, since \textit{the gradient of a function is the slope of its best local linear approximation}, 
it is reasonable to just use the slope of the fitted linear model, a.k.a., $\boldbeta(\boldx^\star)$, to approximate the gradient $\nabla g(\boldx^\star)$. See Figure \ref{fig:local-linear} in Appendix for an illustration.

We apply the same rationale to estimate $\nabla (h\circ r)(\boldx^\star)$. 
However, 
unlike local linear interpolation, we cannot evaluate $h \circ r$  at any input. 
Thus, we cannot directly use the least squares objective \eqref{eq.local.linear.obj} to obtain a local linear interpolation.
Similar to what we have done in Section \ref{sec.nw.interpo}, we look for a tractable population estimator for estimating $h \circ r$, which can be approximated using samples from $p$ and $q$. Then, we ``convert it'' into a local linear objective. 

In the following section, we derive an objective for estimating $h \circ r$ by maximizing a variational lower bound of a \textit{mirror divergence}.

\subsection{Mirror Divergence}

\begin{definition}
\label{def.dual.div}

Let $D_\phi[p, q]$ and $D_\psi[p, q]$ denote two $f$-divergences with $f$ being $\phi $ and $ \psi$ respectively. 
$D_\psi$ is the mirror of $D_\phi$ if and only if $\psi'(r) \triangleq r\phi'(r) - \phi(r)$, where $\triangleq$ means equal up to a constant. 
\end{definition}

For example, let $D_\phi[p, q]$ and $D_\psi[p, q]$ be $\mathrm{KL}[p,q]$ and $\chi^2_\mathrm{p}[p,q]$ respectively. From Table \ref{table:f_div_ode}, we can see that  $\phi(r) = r\log r$ and $\psi(r) = \frac{1}{2}(r - 1)^2$. 
Thus $r\phi'(r) - \phi(r) = r(1+\log r) - r \log r = r  \triangleq \psi'(r) = r - 1$. Therefore, $\chi^2_\mathrm{p}[p,q]$ is the mirror of $\mathrm{KL}[p,q]$. Similarly, we can verify that
$\mathrm{KL}[p, q]$ is the mirror of $\mathrm{KL}[q, p]$ and $\mathrm{KL}[q, p]$ is the mirror of $\chi^2_\mathrm{n}[p,q]$. In general, $D_\psi[p, q]$ is the mirror of $D_\phi[p, q]$ does not imply the other direction. The mirror divergence is also not unique. 

Here, we list a few more examples of $f$-divergences and their mirror divergences:

\begin{itemize}
    \item Jensen-Shannon Divergence: 
    \begin{align*}
        \phi = \frac{1}{2}{r\log(r) - (r + 1)\log\left(\frac{r+1}{2}\right)}. 
    \end{align*}
    \begin{itemize}
        \item Mirror Divergence: 
        \begin{align*}
            \psi & = \frac{\log\left(r+1\right)}{2}-\frac{\log\left(2\right)}{2}+\\
        &r\,\left(\frac{\log\left(\frac{r}{2}+\frac{1}{2}\right)}{2}-\frac{1}{2}\right)+\frac{1}{2}.
        \end{align*}
    \end{itemize}
    \item Squared Hellinger Distance: 
    $\phi = \frac{{\left(\sqrt{r}-1\right)}^2}{2}$.
    \begin{itemize}
        \item Mirror Divergence: $\psi = \frac{r^{3/2}}{3} - \frac{r}{2}$. 
    \end{itemize}
    \item Total Variational Distance: $\phi = \frac{1}{2}|r - 1|$. 
    \begin{itemize}
        \item Mirror Divergence: $\psi = \frac{1}{2}|r - 1|$. 
    \end{itemize}
\end{itemize}
We also provide a MATLAB script in the GitHub repo to compute the mirror divergence of an $f$-divergence.

\subsection{Gradient Estimator using Local Linear Interpolation}
\label{sec.unbiased.est} 

The key observation that helps derive a tractable objective is that $h \circ r$ is the $\argmax$ of ``the mirror variational lowerbound''. 
Suppose $h$ is associated with an $f$-divergence $D_\phi$ as per Theorem \ref{thm.gf} and $D_\psi$ is the mirror of $D_\phi$. 
Then $h \circ r$ is the $\argmax$ of the following objective: 
\begin{align}
\label{eq.fdiv.variational}
    D_\psi[p,q] = \max_d \mathbbE_p[d(\boldx)] - \mathbbE_q[\psi_\mathrm{con} (d(\boldx))],
\end{align}
where $\psi_\text{con}$ is the \textit{convex conjugate} of $\psi$. The formal statement and its proof can be found in Appendix \ref{sec.var.obj.proof}. 
The equality in \eqref{eq.fdiv.variational} is known in previous literature \citep{Nguyen2010,Nowozin2016b} and the objective in \eqref{eq.fdiv.variational} is commonly referred to as the variational lowerbound of $D_\psi[p,q]$. 

Notice that the expectations in \eqref{eq.fdiv.variational} can be approximated by $\widehat{\mathbbE}_p[\cdot]$ and $\widehat{\mathbbE}_q[\cdot]$ using samples from $p$ and $q$ respectively. 

This is a \textit{surprising result}.  Since $h$ is related to $D_\phi$'s field (as per Theorem \ref{thm.gf}), one may associate maximizing $D_\phi$'s variational lowerbound with its velocity field estimation. However, the above observation shows that, to approximate $D_\phi$'s field, one should maximize the variational lowerbound of  its \textit{mirror divergence} $D_\psi$!  To our best knowledge, this ``mirror structure'' in the context of WGF has never been studied before.

We then localize \eqref{eq.fdiv.variational} to obtain a local linear estimator of $h \circ r$ at a fixed point $\boldx^\star$. First, we parameterize the function $d$ using a linear model $d_{\boldw, b}(\boldx) := \langle \boldw, \boldx \rangle + b $. Second, we weight the objective using  $k_\sigma(\boldx, \boldx^\star)$, which leads to the following local linear objective:
\begin{align}
\label{eq.kernelize.hrestimator}
 \left(\boldw(\boldx^\star), b(\boldx^\star) \right)  = &\argmax_{\boldw \in \mathbbR^d, b \in \mathbbR} \ell(\boldw, b; \boldx^\star), \notag \\
 \text{where } \ell(\boldw, b; \boldx^\star)  
:= & \hatE_p[k_\sigma^\star d_{\boldw, b}(\boldx)] -    \hatE_q[k_\sigma^\star \psi_\text{con}(d_{\boldw, b}(\boldx))]
\end{align}
The above transformation is similar to how the local linear regression ``localizes'' the ordinary least squares objective.

Solving \eqref{eq.kernelize.hrestimator}, we get a  linear approximation of $h \circ r$ at $\boldx^\star$:   $$h(r(\boldx^\star)) \approx  \langle \boldw(\boldx^\star), \boldx^\star \rangle + b(\boldx^\star).$$ 
Following the intuition that $\nabla (h\circ r)(\boldx^\star)$ is the slope of the best local linear fit of $h(r(\boldx^\star))$, 
we use $\boldw(\boldx^\star)$ to approximate $\nabla (h\circ r)(\boldx^\star)$. We will theoretically justify this approximation in Section \ref{sec.asy.unbiasedness}. 
Now let us study two examples: 

\begin{example}
    Suppose we would like to estimate $\mathrm{KL}[p,q]$'s field at $\boldx^\star$, which is $\nabla r(\boldx^\star)$. 
    Using Definition \ref{def.dual.div},
    we can verify that the mirror of $\mathrm{KL}[p,q]$ is $D_\psi = \chi^2_\mathrm{p}[p,q]$, in which case
    $\psi = \frac{1}{2}(r-1)^2$. The convex conjugate of $\psi$ is $\psi_\mathrm{con}(d)= d^2/2 + d$.
    Substituting $\psi_\mathrm{con}$ in \eqref{eq.kernelize.hrestimator}, 
    the gradient estimator $\boldw_{\rightarrow}(\boldx^\star) \approx \nabla r(\boldx^\star)$ is obtained by the following objective:
\begin{align}
\label{eq.obj.thetax0.ls}
     (\boldw_{\rightarrow}(\boldx^\star), b_{{\rightarrow}}(\boldx^\star)) 
    :=&  \argmax_{\boldw \in \mathbb{R}^d, b \in \mathbbR} \ell_\rightarrow(\boldw, b; \boldx^\star) \notag \\ \text{where } \ell_\rightarrow(\boldw, b; \boldx^\star) := &\hatE_p[k_\sigma^\star \cdot d_{\boldw, b}(\boldx)] - \notag \\
    & \hatE_q\left[k_\sigma^\star  \cdot \left(\frac{d_{\boldw, b}(\boldx)^2}{2} + d_{\boldw, b}(\boldx)\right)\right].
\end{align}  
\end{example}
\begin{example}
    Suppose we would like to estimate $\mathrm{KL}[q,p]$'s field at $\boldx^\star$, which is $\nabla \log r(\boldx^\star)$. 
    Using Definition \ref{def.dual.div},
    we can verify that the mirror of $\mathrm{KL}[q,p]$ is $D_\psi = \mathrm{KL}[p,q]$, in which case,  
    $\psi(r) = r\log r$. 
    The convex conjugate  of $\psi$ is $\psi_\mathrm{con}(d) = \exp(d-1)$. The gradient estimator $\boldw_{{\leftarrow}}(\boldx^\star) \approx \nabla \log r(\boldx^\star)$ is obtained by the following objective:
\begin{align}
\label{eq.obj.thetax0.kl}
    (\boldw_{{\leftarrow}}(\boldx^\star), b_{\leftarrow}(\boldx^\star)) := &\argmax_{\boldw \in \mathbb{R}^d, b \in \mathbbR}  \ell_\leftarrow(\boldw, b; \boldx^\star)\notag \\
    \text{where }\ell_\leftarrow(\boldw, b; \boldx^\star) := & \hatE_p[k_\sigma^\star \cdot  d_{\boldw, b}(\boldx)] - \notag \\
    & \hatE_q[k_\sigma^\star \cdot \exp(d_{\boldw, b}(\boldx) - 1)].
\end{align}  
\end{example}

In the following section, we show that 
the estimation error $\|\boldw(\boldx^\star) - \nabla (h\circ r)(\boldx^\star)\|$ vanishes as $\sigma \to 0$ and $n \to \infty$.

\subsection{Effectiveness of Local Linear Interpolation}
\label{sec.asy.unbiasedness}
In this section, we state our main theoretical result.
Let $(\boldw(\boldx^\star), b(\boldx^\star))$ be
a stationary point of $\ell(\boldw, b; \boldx^\star)$.
We denote the domain of $\boldx$ as $\mathcal{X}$ (not necessarily $\mathbbR^d$). Without loss of generality, we also assume all sample averages $\hatE[\cdot]$ are averaged over $n$ samples.  
We prove that, 
 $\boldw(\boldx^\star)$ is a consistent estimate of  $\nabla (h\circ r)(\boldx^\star)$ assuming the change rate of the flow is bounded. 
\begin{assumption}
\label{ass.curvature}
    The change rate of the velocity fields is well-behaved, i.e., \[\sup_{\boldx \in \mathcal{X}} \|\nabla^2 (h\circ r)(\boldx)\| \le \kappa. \]
\end{assumption}
This is an analogue of the assumption on $\nabla^2 \log r$ in Proposition \ref{svgd.bias}. 

\begin{assumption}
\label{ass.nw2}
There exists a constant $C_k > 0$ independent of $\sigma$, 
    \begin{align*}
\frac{1}{2} \int q(\sigma\boldy + \boldx^\star) k(\boldy) \cdot \left \|\boldy \right\|^2 \cdot \left \|[\sigma \boldy + {\boldx^\star}, 1]\right \| \dy \le  C_k.
\end{align*}
\end{assumption}
This assumption is similar to the first inequality in \eqref{eq.nw.condition}. Expanding $q(\sigma\boldy + \boldx^\star)$ using Taylor expansion and providing that our kernel is well behaved, this assumption essentially implies the boundedness of the $q(\boldx^\star)$ and $\|\nabla^2 q(\boldx^\star)\|$. 

Define two shorthands:
$
\boldw^* := \nabla (h\circ r)(\boldx^\star)  \text{ and } b^* := h(r(\boldx^\star)) - \langle \nabla (h\circ r)(\boldx^\star), \boldx^\star\rangle.
$
\begin{assumption}
Let $\tilde{\boldx} := \left[\boldx^\top, 1 \right]^\top$. As $\sigma \to 0$, 
\label{ass.boundedcov}
        \begin{align*}
    &\mathrm{tr}\left[{\mathrm{Cov}_p\left[\frac{1}{\sigma^d}{k_\sigma^\star \cdot \tilde{\boldx}}\right]}\right] = O(\frac{1}{{\sigma^d}}), \\
        & \mathrm{tr}\left[{\mathrm{Cov}_q\left[\frac{1}{\sigma^d}{k_\sigma^\star \cdot \psi_\mathrm{con}'(\langle \boldw^*, \boldx\rangle + b^*)\tilde{\boldx}}\right]}\right]  = O(\frac{1}{{\sigma^d}}). 
    \end{align*}
\end{assumption}
These two are analogues of \eqref{eq.vars}. They are required so that our sample approximation of the objective is valid and concentration inequalities can be applied.

\begin{assumption}
\label{ass.psi.twice}
For all $a \in [0,1]$ and $\boldx \in \mathcal{X}$, $\psi''_\mathrm{con}\left[ah(r(\boldx)) + (1-a) (\langle \boldw^*, \boldx \rangle + b^*) \right] \le C_{\psi''_\mathrm{con}}. $
\end{assumption}
This is a unique assumption to our estimator, where we assume that the convex conjugate $\psi_\mathrm{con}$ has a bounded second-order derivative.

\begin{theorem}
\label{them.unbiasedness}

Suppose Assumption \ref{ass.curvature}, \ref{ass.nw2}, \ref{ass.boundedcov} and \ref{ass.psi.twice} holds. 
    If there exist strictly positive constants $W, B, \Lambda_{\mathrm{min}}$ that are independent of $\sigma$ and $n$ such that, 
    \begin{align}
        \label{eq.upperbounding.condition}
        &\|\boldw^*\| \le W,  ~~~ |b^*| \le B
    \end{align}
    and for all $\boldw \in  \{\boldw | \|\boldw\| < 2W\}$ and $b \in \{b | |b| < 2B\}$, 
    \begin{align}
\lambda_\mathrm{min}\left\{\hatE_q\left[k_\sigma^\star \nabla_{[\boldw, b]}^2 \psi_\mathrm{con}(\langle \boldw, \boldx \rangle + b)\right]\right\} \ge \sigma^d \Lambda_\mathrm{min},
\label{eq.eigenvalu.condition} 
    \end{align}
    holds with high probability. 
    Then there exists $\sigma_0, N, K > 0$ such that  for all  $0 < \sigma < \sigma_0, n > N $,  
    \[\|\boldw(\boldx^\star) - \boldw^*\| \le \frac{\frac{K}{ \sqrt{n\sigma^d}}+ \kappa C_k C_{\psi_\mathrm{con}''} \sigma^2}{\Lambda_\mathrm{min}},\]
    with high probability.
\end{theorem}
The proof can be found in Appendix \ref{sec.proof.bias}. 

Similar to Proposition \ref{svgd.bias}, the estimation error is upperbounded by the sample approximation error that reduces with the ``effective sample size'' $n\sigma^d$, and a bias term that reduces with $\sigma$. 
Interestingly, although we make the smoothness assumption on $\nabla^2 (h\circ r)$, similar to Proposition \ref{svgd.bias}, the bias vanishes at a \textit{quadratic rate} $\sigma^2$, unlike the linear rate obtained in Proposition \ref{svgd.bias}. 

We can deduce that, with an optimal choice of $\sigma \sim n^{-1/(d+4)}$, the estimation error is $O_p(n^{-2/(d+4)})$, a rate faster than the rate implied by Proposition \ref{svgd.bias}.

Using Theorem \ref{them.unbiasedness}, we can prove the consistency of various velocity field estimators for different $f$-divergences.

\begin{corollary}
\label{co.gradratio} 
    Suppose $\sup_{\boldx \in \mathcal{X}} \|\nabla^2 r(\boldx)\| \le \kappa_\rightarrow$, Assumption  \ref{ass.nw2}, \ref{ass.boundedcov} holds and
    \begin{align}
\lambda_\mathrm{min}&\left\{\hatE_q\left[k_\sigma(\boldx, \boldx^\star )\tilde{\boldx} \tilde{\boldx}^\top\right]\right\} \ge \sigma^d \cdot \Lambda_\rightarrow >0, \label{eq.cond.hessian}
    \end{align}    
    holds with high probability.
    Then $\exists \sigma_0, N, K > 0$, 
    $$ \|\boldw_{\rightarrow
    }(\boldx^\star) - \nabla r(\boldx^\star)\| \le \frac{\frac{K}{\sqrt{n\sigma^d}} + \kappa_\rightarrow \cdot C_k \cdot  \sigma^2}{\Lambda_\rightarrow}$$ for all  $0 < \sigma < \sigma_0, n > N $ holds with high probability. 
\end{corollary}
See Appendix \ref{proof.col.ulsif} for the proof. 

\begin{corollary}
\label{co.gradlogratio}
    Suppose $\sup_{\boldx \in \mathcal{X}}\|\nabla^2 \log r(\boldx)\| \le \kappa_\mathrm{\leftarrow}$, and Assumption \ref{ass.nw2}, \ref{ass.boundedcov} holds. 
    Assume  
    \[\sup_{\boldx \in \mathcal{X}}\|\boldx\| \le C_\mathcal{X}, \sup_{\boldx \in \mathcal{X}}\log r(\boldx) - 1 \le C_{\log r}, \]  and there exists $B, W < \infty$, so that \[\|\nabla \log r(\boldx^\star)\| < W, |\log r(\boldx^\star) - \langle \nabla \log r(\boldx^\star), \boldx^\star \rangle | < B.\] Additionally, 
    \begin{align} 
    \lambda_\mathrm{min}&\left[\hatE_q[k_\sigma(\boldx, \boldx^\star )\tilde{\boldx} \tilde{\boldx}^\top]\right] \ge \sigma^d \cdot \Lambda_\leftarrow >0, \label{eq.cond.eigen}
    \end{align}
    holds with high probability. 
    Then $\exists \sigma_0, N, K > 0$ and for all  $0 < \sigma < \sigma_0, n > N, $  
    \begin{align*} 
        &\|\boldw_{\mathrm{\leftarrow}}(\boldx^\star) - \nabla \log r(\boldx^\star)\| \\
        \le&  \frac{1}{(\Lambda_\leftarrow \cdot   \exp(-2WC_\mathcal{X} - 2B -1))} \cdot \frac{K}{\sqrt{n\sigma^d}} + \\ 
       & \frac{\kappa_\leftarrow \cdot C_k \cdot \left [\exp(W C_\mathcal{X} + B -1) \vee  \exp(C_{\log r} - 1) \right] \cdot  \sigma^2 }{ \Lambda_\leftarrow \cdot   \exp(-2WC_\mathcal{X} - 2B -1)},
    \end{align*}
    holds with high probability. 
\end{corollary}
See Appendix \ref{proof.co.logratio} for the proof. 
Note that due to the assumption that $\|\boldx\|$ is bounded, Corollary \ref{co.gradlogratio} can only be applied to density ratio functions with bounded input domains. However, this does include important examples such as images, where pixel brightnesses are bounded within $[0,1]$.
Using the same proof techniques, it is possible to derive more Corollaries for other $f$-divergence velocity fields. We leave them as a future investigation.

\subsection{Model Selection via Local Linear Interpolation}
\label{sec.model.selection}
Although Theorem \ref{them.unbiasedness} says the estimation bias disappears as $\sigma \to 0$, 
when we only have a finite number of samples, 
the choice of the kernel bandwidth $\sigma$ controls the bias-variance trade-off of the local estimation. 
Thus we propose a model selection criterion. The details of the procedure is provided in Appendix \ref{sec.mosel}. 

The high level idea is: 
Suppose we have testing samples from $p$ and $q$. A good choice of $\sigma$ would result in a good approximation of $h \circ r$ on \textit{testing points}, thus the best approximation of $h \circ r$ would maximize the variational lower bound \eqref{eq.fdiv.variational}. Therefore, we only need to evaluate \eqref{eq.fdiv.variational} on testing samples to determine the optimality of $\sigma$. 

Our local linear estimator offers a unique advantage of model selection because it is formulated as a non-parametric curve fitting problem. In contrast, SVGD lacks a systematic approach and has to resort to the ``median trick''.

\section{Experiments}
\label{sec.exp}

\subsection{Reducing KL Divergence: SVGD vs. NW vs. Local Linear Estimator}
\label{sec.svgd.vs.nsvgd}
\begin{figure*}
    \centering
\includegraphics[width=.98\textwidth]{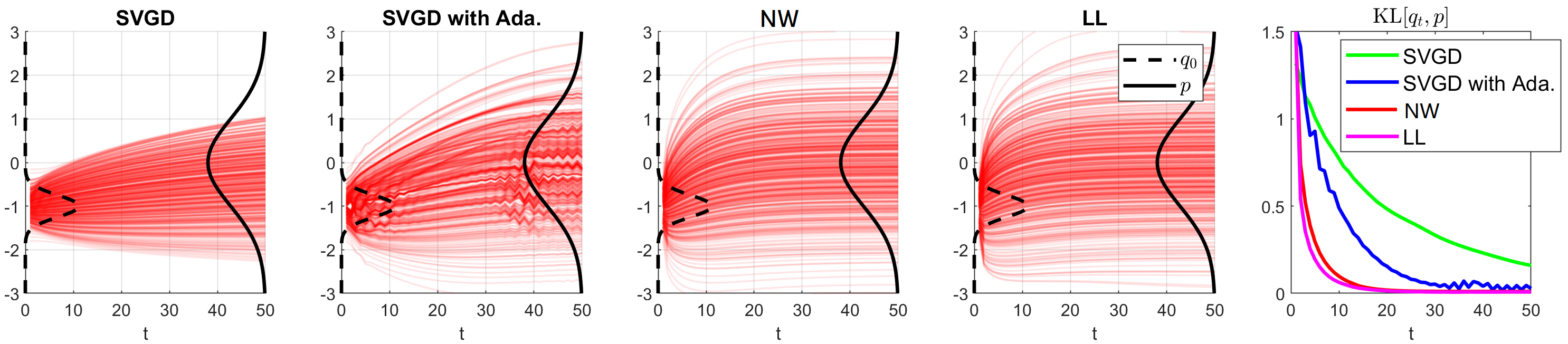}
    \caption{Particle Trajectories of SVGD, SVGD with AdaGrad, NW, LL. Approximated $\mathrm{KL}[q_t, p]$ with different methods.}
    \label{fig.svgd}
\end{figure*}

In this experiment, we investigate the performance of SVGD, NW and Local Linear (LL) estimator through the task of  
 minimizing $\mathrm{KL}[q_t,p]$. 
We let SVGD, NW and LL fit the target distribution $p = \mathcal{N}(0, 1)$. 500 iid initial particles are drawn from $q_0 = N(-1, 0.25^2)$. For all methods, we use naive gradient descent to update particles with a fixed step size 0.1. We also consider a variant of SVGD where AdaGrad \citep{duchi2011adaptive} is applied to adjust the step size dynamically.  
For SVGD and SVGD with AdaGrad, we use the MATLAB code
provided by \citet{LiuQ2016SVGD} with its default heuristics.
We plot the trajectories of particles of all three methods in Figure \ref{fig.svgd}. 

Although all three algorithms move particles toward the target distribution,  the naive SVGD does not spread the particle mass quickly enough to cover the target distribution when using the same step size.
This situation is much improved by applying the adaptive learning rate method (AdaGrad). 
In comparison, NW and LL both converge fast. After 20 iterations, all particles have arrived at the target positions. 
Since all methods are motivated by minimizing $\mathrm{KL}[q_t, p]$, we plot the $\mathrm{KL}[q_t,p]$ approximated by Donsker and Varahan Lower Bound \citep{DonskerVarahan1976} for all four methods. The plot of $\mathrm{KL}[q_t, p]$ agrees with our qualitative assessments:  AdaGrad SVGD can reduce the KL significantly faster than the vanilla SVGD with naive gradient descent. After 20 iterations, the KL divergence for both NW and LL particles reaches zero, indicating that the particles have fully converged to the target distribution. 
LL achieves a performance comparable to NW. This is a remarkable result as NW, and SVGD has access to the true $\nabla \log p$, but LL only has samples from $p$. We also compare NW, LL and SVGD performance at different sample sizes ($n = 100, n = 250$). The results can be found in the Appendix \ref{sec.svgd.diff.n}. 

In the next sections, we will showcase the performance of LL in forward/backward KL minimization problems. 

\subsection{Joint Domain Adaptation}

In domain adaptation, we want to use source domain samples to help a prediction task in a target domain. This addresses situations where the training data for a method may differ from the real data when deployed.
We assume that source samples $\mathcal{D}_q := \left\{ (\boldx_q^{(i)}, y_q^{(i)})\right\}_{i=1}^{n_q}$ are drawn from a joint distribution $\mathbb{Q}_{XY}$ and target samples $\mathcal{D}_p := \left\{ (\boldx_p^{(i)}, y_p^{(i)})\right\}_{i=1}^{n_p}$ are drawn from a different joint distribution $\mathbb{P}_{XY}$. However, $y_p$ is missing from the target set. Thus, we want to predict missing labels in $\mathcal{D}_p$ with the help of $\mathcal{D}_q$.  
\citet{Courty15,courty2017joint} propose to 
find an optimal map that aligns the distribution $\mathbb{Q}_{XY}$ and $\mathbb{P}_{XY}$, then train a classifier on the aligned source samples.
Inspired by this method, we propose to align samples by minimizing $\mathrm{KL}[q_{t},p]$, where $p$ is the  density of the target $\mathbb{P}_{XY}$ and $q_t$ is a particle-label pair distribution whose samples are $\mathcal{D}_{q_t} := \left\{ (\boldx_t^{(i)}, y_q^{(i)})\right\}_{i=1}^{n_q}$.  
To minimize $\mathrm{KL}[q_{t},p]$, we evolve $\boldx_{t}$ according to the backward KL field and $\boldx_{t = 0}$ is initialized to be the source input $\boldx_q$. In words, we transport source input samples so that the transported and target  samples are aligned in terms of minimizing the backward KL divergence. 
After $T$ iterations, we can train a classifier using transported source samples $\{(\boldx_{T}, y_q)\}$ to predict target labels. 

One slight issue is that we do not have labels in the target domain but performing WGF requires joint samples $(X,Y)$. To solve this, we adopt the same approach used in \citet{courty2017joint}, replacing $y_p$ with a proxy $\hat{g}(\boldx_p)$, where $\hat{g}$ is a prediction function trained to minimize an empirical transportation cost (See Section 2.2 in \citet{courty2017joint})\footnote{If we know some labels $y_p$ from the target domain, instead of using the proxy $\hat{g}(\boldx_p)$, we can use pairs ${(\boldx_p, y_p)}$ to align the source and target samples.}.   
We demonstrate our approach in a toy example, in Figure \ref{fig:wgfda}. 

\begin{figure}[H]
    \centering
    \includegraphics[width=.47\textwidth]{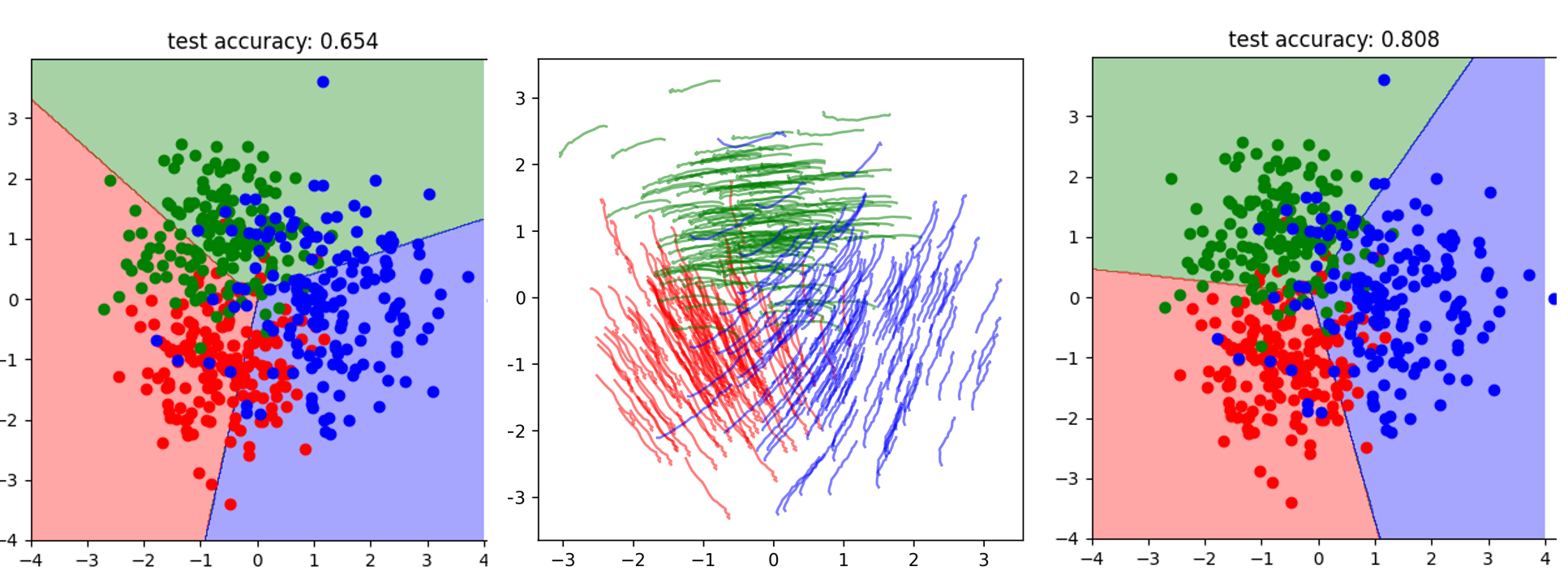}
    \caption{Left: the source classifier (represented by colored areas) misclassifies many testing points (colored dots). Middle: WGF moves particles to align the source and target samples. Lines are trajectories of sample movements in each class. Right: the retrained classifier on the transported source samples gives a much better prediction. }
    \label{fig:wgfda}
\end{figure}

Table \ref{tab:domain_transfer} compares the performance of adapted classifiers on a real-world 10-class classification dataset named ``office-caltech-10'', where images of the same objects are taken from four different domains (amazon, caltech, dslr and webcam). We reduce the dimensionality by projecting all samples to a 100-dimensional subspace using PCA. 
We compare the performance of 
the \textbf{base} (the source RBF kernel SVM classifier),
the Joint Distribution Optimal Transport \citep{courty2017joint} (\textbf{JDOT}), an RBF kernel SVM trained on the WGF transported source samples $\{(\boldx_{T}, y_q)\}$ (\textbf{WGF}) and an SVM trained on MMDFlow \citep{hagemann2024posterior} transported source samples (\textbf{MMD}). The classification accuracy on the entire target sets are reported. 
It can be seen that in some cases, reusing the source classifiers in the target domain does lead to catastrophic results (e.g. amazon to dslr, caltech to dslr). However, we can avoid such performance decline by using any joint distribution-based domain adaptation. 
It can be seen that both \textbf{WGF} and \textbf{MMD} achieve superior performance compared to \textbf{JDOT} while \textbf{WGF} has the best performance in most adaptation cases.

\begin{table}[t]
\centering
\begin{tabular}{|c|c|c|c|c|}
\hline
\textbf{$\mathcal{D}_q \to \mathcal{D}_p$} & \textbf{base} & \textbf{JDOT} & \textbf{WGF} & \textbf{MMD} \\
\hline
amz. $\to$ dslr & 27.39\% & 65.61\% & \textbf{78.34}\% & \textbf{78.34}\% \\
amz. $\to$ web. & 61.69\% & 67.80\% & 84.07\% & \textbf{89.15}\% \\
amz. $\to$ cal. & 81.66\% & 63.58\% & \textbf{82.72}\% & 82.19\% \\
dslr $\to$ amz. & 70.35\% & 72.96\% & \textbf{85.91}\% & 76.30\% \\
dslr $\to$ web. & 94.92\% & 76.61\% & \textbf{95.25}\% & 86.10\% \\
dslr $\to$ cal. & 58.95\% & 72.31\% & \textbf{79.25}\% & 69.01\% \\
web. $\to$ amz. & 75.78\% & 75.37\% & \textbf{91.34}\% & 89.46\% \\
web. $\to$ dslr & 94.27\% & 73.25\% & 98.73\% & \textbf{100.00}\% \\
web. $\to$ cal. & 67.05\% & 63.49\% & \textbf{78.09}\% & 75.16\% \\
cal. $\to$ amz. & 83.40\% & 84.55\% & \textbf{91.13}\% & 89.04\% \\
cal. $\to$ dslr & 26.11\% & 69.43\% & \textbf{84.71}\% & \textbf{84.71}\% \\
cal. $\to$ web. & 63.39\% & 74.58\% & \textbf{80.34}\% & 79.32\% \\
\hline
\end{tabular}
\caption{Domain adaptation of different domains in Office-Caltech-10 dataset.}
\label{tab:domain_transfer}
\end{table}

\subsection{Missing Data Imputation}
\begin{figure*}[t]
    \centering
    \includegraphics[width=.95\textwidth]{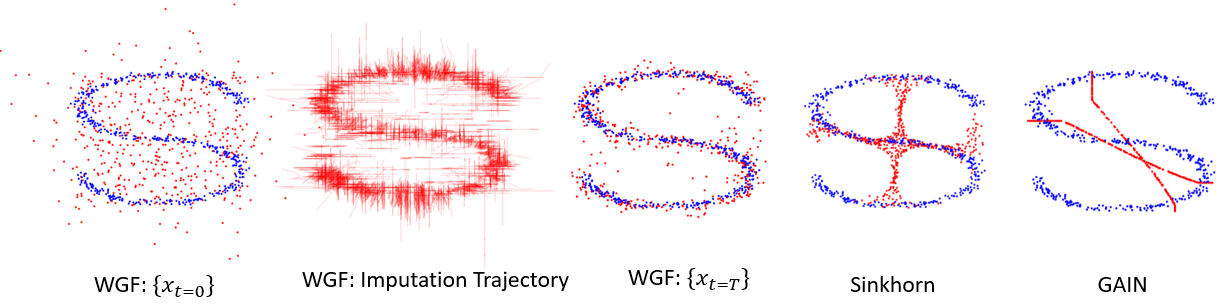}
    \caption{Comparison of imputation methods. Fully observed samples are plotted in blue, and imputed samples in red. The leftmost plot shows the initial particles in the WGF impute. The second left plot visualizes the imputation trajectories of different particles. The third left plot is the final output after 100 WGF iterations. }
    \label{fig.imp}
\end{figure*}
In missing data imputation, we are given a joint dataset $\tilde{\mathcal{D}} := \{ (\tilde{\boldx}^{(i)}, \boldm^{(i)}) \}$, where $\boldm \in \{0,1\}^d$ is a mask vector and $m_j^{(i)} = 0$ indicates the $j$-th dimension of $\tilde{\boldx}^{(i)}$ is missing. The task is to ``guess'' the missing values in $\tilde{\boldx}$ vector. In recent years, GAN-based missing value imputation, e.g., \textbf{GAIN} \citep{yoon2018gain}, has gained significant attention \citep{zhang_systematic_2023}. 
Let $\boldx$ be an imputation of $\tilde{\boldx}$. The basic idea is that if $\boldx$ is a perfect imputation of $\tilde{\boldx}$, then a classifier cannot predict $m_j$ given $(\boldx, \boldm_{-j})$. For example, given a perfectly imputed image, one cannot tell which pixels are imputed and which pixels are observed. 
Therefore, in their approach, a generative network is trained to minimize the aforementioned classification accuracy. In fact, this method teaches the generator to \textit{break the dependency between} $\boldx$ and $\boldm$.
Inspired by this idea, we propose to impute $\{\tilde{\boldx}\}$ by iteratively updating particles $\{\boldx_t\}$. The initial particle $\boldx_{t=0}$ is set to be
\begin{align*}
    x_{0,j} = \begin{cases}
        \tilde{x}_j & m_j = 1\\
        \text{Sample from e.g., } \mathcal{N}(0,1) & m_j = 0.
    \end{cases}
\end{align*}
After that, the particles $\{\boldx_t\}$ are evolved according to the forward KL field that minimizes $\mathrm{KL}[p_{X_tM}, p_{X_t}p_{M}]$, i.e., the mutual information between $X_t$ (particles) and $M$ (mask). 
Note that we only update missing dimensions, i.e.,  
\begin{align*}
    x_{t+1,j} := \begin{cases}
        x_{t, j} & m_j = 1\\
        x_{t,j} + \eta \partial_{x_j} r_t(\boldx_t, \boldm) & m_j = 0. 
    \end{cases}
\end{align*}
Samples from $p_{X_tM}$ are available to us since we observe the pairs $\{(\boldx, \boldm)\}$ and samples from $p_{X_t}p_{M}$ can be constructed as $\{(\boldx, \boldm')\}$, where $\boldm'$ is a random sample of $M$ given the missing pattern. 

In the first experiment, we test the performance of various imputers on an ``S''-shaped dataset, where samples are Missing Completely at Random (MCAR) \citep{rubin_inference_1976}. The results are plotted in Figure \ref{fig.imp}. We compare our imputed results (\textbf{WGF}) with Optimal Transport-based imputation method (\textbf{Sinkhorn}) \citep{muzellec2020missing}, and GAN-based imputation method (\textbf{GAIN}) \citep{yoon2018gain}.
$\sigma$ is chosen by automatic model selection described in Section \ref{sec.mosel}. It can be seen that our imputer, based on minimizing $\mathrm{KL}[p_{X_tM},p_{X_t}p_{M}]$  nicely recovers the ``S''-shape after 100 particle update iterations. However, Sinkhorn and GAIN imputation methods struggle to accurately restore the ``S''-shaped pattern.

\begin{figure}[t]
\centering
    \includegraphics[width=.4\textwidth]{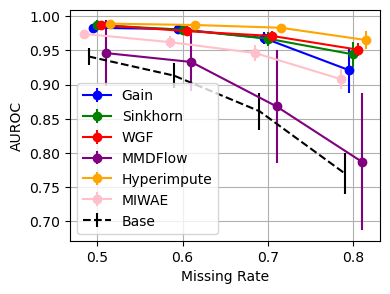}
    \caption{AUROC of a linear SVM classifier on the imputed Breast Cancer dataset. \textbf{Base} indicates the performance of a baseline imputer where we impute the missing values with Gaussian noises. }
    \label{fig.breast}
\end{figure}

In the second experiment, we test the performance of our algorithm on a real-world Breast Cancer classification dataset \citep{misc_breast_cancer_14} in Figure \ref{fig.breast}. 
This is a 30-dimensional binary classification dataset and we artificially create missing values by following the MCAR paradigm with different missing rates. 
Since the dataset is a binary classification dataset, 
we compare the performance of linear SVM classifiers trained on imputed datasets. The performance is measured by the Area Under the ROC Curve (AUROC) on hold-out testing sets. 
In addition to \textbf{Sinkhorn} and \textbf{GAIN}, we validate the performance of our method with three additional algorithms: MMD flow using negative distance kernel (\textbf{MMDFlow}) \citep{hagemann2024posterior}, a model selection-based method \textbf{HyperImpute} \citep{Jarrett2022HyperImpute}, and an auto encoder-based approach (\textbf{MIWAE}) \citep{pmlr-v97-mattei19a}.
The result shows that SVM trained on the dataset imputed by our method achieves comparable performance to datasets imputed by the other benchmark methods. Our method is robust against the choice of $\sigma$. 
Details and the selection of hyperparameters can be found in Section \ref{sec.missing.exp}.

\section{Limitations}
All $f$-divergence fields are defined using the density ratio function $p/{q_t}$. However, the ratio function is not defined when the input domains of $p$ and $q_t$ are non-overlapping. This problem is particularly noticeable when working on high-dimensional, real-world datasets (such as images). This so-called ``density chasm problem'' \cite{rhodes2020telescoping} will also affect the velocity field estimation as assumptions required by our estimators e.g., Assumption \ref{ass.curvature} or \eqref{eq.upperbounding.condition} depends on the density ratio. One possible solution to this issue is to build a ``bridge'' using interpolations of two distributions, estimate the density ratios of ``neighbouring'' interpolations, and then combine these estimates. However, this will significantly increase the computational cost as we need to estimate the density ratio gradients of many pairs of distributions. Developing an efficient gradient estimator using target and particle distribution interpolation is an interesting future task. Our assumptions in Proposition \ref{svgd.bias} and Theorem \ref{them.unbiasedness} indicate that the algorithm's effectiveness depends on the boundedness of $\sup \|\nabla^2 (h \circ r)\|$. This assumption is not necessarily true for WGF of some divergence over the entire real domain (e.g. See Appendix G.). This restriction suggests in some applications, these flows are non-ideal choices.

Another potential limitation of the proposed estimator is its applicability to high-dimensional datasets (such as the ImageNet dataset \citep{Li2009}). Although some preliminary investigations show that the proposed estimator does work reasonably well on some high-dimensional generative tasks (see Section \ref{sec.wgf}), it is known that local regression tends to perform poorly on high-dimensional datasets (see, e.g., \citep{stone1980,Stone1982}). Instead of performing WGF updates directly on the high-dimensional datasets, we can consider performing the updates in a low-dimensional feature space (see Section \ref{sec.wgf.fea}). Some preliminary investigations have shown promising results (see Section \ref{sec.exp.fea}).   

Finally, our method does not have a generator in the form of a neural network (like the one used by \citep{gao2019deep}), so we can only rely on the current particle set to estimate the vector field. This means that the sample size must be fixed before the flow. Thus, our method is better suited for applications such as domain adaptation and data imputation, where the number of samples to be transported is fixed before running the flow.

\section{Conclusion}
In recent years, it has been discovered that by iteratively updating a set of particles, WGF can approximate a target distribution by reducing the $f$-divergences between corresponding distributions. 
This paper addresses the important problem of estimating the velocity fields induced by various $f$-divergence WGFs. 
We propose novel interpolation-based estimators for different $f$-divergences fields and prove their validity. We demonstrate the effectiveness of these estimators through two novel applications: domain adaptation and missing data imputation. Our results show that even without access to the density ratio function, the velocity fields can be efficiently estimated from samples of the target and particle distributions.

\section*{Acknowledgments}
We thank four anonymous reviewers for their feedback on our paper. 
JS was supported by the EPSRC
Centre for Doctoral Training in Computational Statistics
and Data Science, grant number EP/S023569/1.

\section*{Impact Statement}
This paper presents work whose goal is to advance the field of Machine Learning. There are many potential societal consequences of our work, none of which we feel must be specifically highlighted here.

\bibliography{imdb}
\bibliographystyle{plainnat}

\appendix

\onecolumn

\section{Proof of Proposition \ref{svgd.bias}}
\label{sec.svgd.proof}
\begin{proof} 
    \begin{align}
    \left | \hat{u}_i(\boldx^\star) - \partial_i \log r(\boldx^\star) \right |
    = & \left |  \frac{\widehat{\mathbbE}_q\left[k_\sigma(\boldx, \boldx^\star) \partial_i \log r(\boldx)\right]}{\hatE_q[k_\sigma(\boldx, \boldx^\star)]} - \partial_i \log r(\boldx^\star) \right | \notag \\
    =& \left | \frac{\hatE_q[k_\sigma(\boldx, \boldx^\star)\left(\partial_i \log r(\boldx^\star) + \langle \nabla\partial_i \log r(\bar{\boldx}), \boldx - \boldx^\star\rangle\right)]}{\hatE_q[k_\sigma(\boldx, \boldx^\star)]} - \partial_i \log r(\boldx^\star)  \right| \notag \\
    = &  \left | \frac{\hatE_q[k_\sigma(\boldx, \boldx^\star) \langle \nabla \partial_i \log r(\bar{\boldx}), \boldx - \boldx^\star\rangle ]}{\hatE_q[k_\sigma(\boldx, \boldx^\star)]}\right| \notag \\
    \le & \frac{\hatE_q[k_\sigma(\boldx, \boldx^\star) \|\nabla \partial_i \log r(\bar{\boldx})\| \cdot \| \boldx - \boldx^\star\| ]}{\hatE_q[k_\sigma(\boldx, \boldx^\star)]} \notag \\
    \le & \kappa \cdot \frac{\hatE_q[k_\sigma(\boldx, \boldx^\star) \| \boldx - \boldx^\star\| ]}{\hatE_q[k_\sigma(\boldx, \boldx^\star)]} \notag \\
    \le & \kappa \cdot \left(\frac{  \hatE_q[k_\sigma(\boldx, \boldx^\star) \| \boldx - \boldx^\star\| ]}{\hatE_q[k_\sigma(\boldx, \boldx^\star)]} - \frac{ \mathbbE_q[k_\sigma(\boldx, \boldx^\star) \| \boldx - \boldx^\star\| ]}{\mathbbE_q[k_\sigma(\boldx, \boldx^\star)]} + \frac{  \mathbbE_q[k_\sigma(\boldx, \boldx^\star) \| \boldx - \boldx^\star\| ]}{\mathbbE_q[k_\sigma(\boldx, \boldx^\star)]}\right) \label{eq.nw.proof}
\end{align}
The second line is due to the mean value theorem and $\bar{\boldx}$ is a point in between $\boldx$ and $\boldx^\star$ in a coordinate-wise fashion. The second inequality is due to the operator norm of a matrix is always greater than a row/column norm.  

Since $\mathrm{Var}_q[\frac{1}{\sigma^d}k_\sigma(\boldx, \boldx^\star) \| \boldx - \boldx^\star\|] = O(\frac{1}{{\sigma^d}})$ due to the assumption, using Chebyshev inequality $$\hatE_q \left[\frac{1}{\sigma^d}k_\sigma(\boldx, \boldx^\star) \| \boldx - \boldx^\star\|\right] - \mathbbE_q \left[\frac{1}{\sigma^d}k_\sigma(\boldx, \boldx^\star) \| \boldx - \boldx^\star\|\right] = O_p(\frac{1}{\sqrt{n\sigma^d}}).$$ Similarly, due to the assumption that 
$\mathrm{Var}_q[\frac{1}{\sigma^d}k_\sigma(\boldx, \boldx^\star)] = O(\frac{1}{{\sigma^d}})$, we have $$\hatE_q \left[\frac{1}{\sigma^d}k_\sigma(\boldx, \boldx^\star) \right] - \mathbbE_q \left[\frac{1}{\sigma^d}k_\sigma(\boldx, \boldx^\star) \right] = O_p(\frac{1}{\sqrt{n\sigma^d}}).$$

\begin{lemma}
\label{lem.aaba}
Suppose $|\mathbbE A| \le A_1 < \infty, |\mathbbE B| > B_0 > 0$. $\hatE A  - \mathbbE{A} = O_p(\frac{1}{\sqrt{n\sigma^d}})$ and $\hatE B  - \mathbbE{B} = O_p(\frac{1}{\sqrt{n\sigma^d}})$
\begin{align*}
    \frac{\hatE A }{\hatE B } 
    - \frac{\mathbbE{A}}{{\mathbbE{B}}} = O_p\left(\frac{1}{\sqrt{n\sigma^d}}\right).
\end{align*}
\end{lemma}
\textbf{Sketch of proof}

    The argument below resembles the proof technique used in A.2 in \citet{givens23a}. Here, we provide a sketch argument. 
    
    $\left|\frac{\hatE A }{\hatE B } 
    - \frac{\mathbbE{A}}{{\mathbbE{B}}} \right| = \left|\frac{\hatE A  \mathbbE B - \mathbbE A \hatE B }{\hatE B \mathbbE{B}}  \right| \le \left|\frac{\hatE A   - \mathbbE A }{\hat{\mathbbE}{B}}  \right| +  \left|\frac{\hatE B   - \mathbbE B }{\hatE B \mathbbE{B}}  \right| \cdot  |\mathbbE A|  $. 
    
    First, due to the boundedness of $|\mathbbE B|$, 
    for any $\epsilon$, $\exists N, \forall n > N$,  $|\hatE B | > |\mathbbE B/2| > B_0/2$ with a probability $1 - \epsilon$. Thus, $\hatE A  - \mathbbE{A} = O_p(\frac{1}{\sqrt{n\sigma^d}})$ implies 
    $\left|\frac{\hatE A   - \mathbbE A }{\hat{\mathbbE}{B}}  \right| = O_p(\frac{1}{\sqrt{n\sigma^d}})$. 
    
    Second, due to the boundedness of $|\mathbbE A|$, and the boundedness of $|\hatE B|$ argued above, $\hatE A  - \mathbbE{A} = O_p(\frac{1}{\sqrt{n\sigma^d}})$ implies $ \left|\frac{\hatE B   - \mathbbE B }{\hatE B \mathbbE{B}}  \right| \cdot  |\mathbbE A| = O_p(\frac{1}{\sqrt{n\sigma^d}})$. 
    
    Hence, $    \frac{\hatE A }{\hatE B } 
    - \frac{\mathbbE{A}}{{\mathbbE{B}}} = O_p\left(\frac{1}{\sqrt{n\sigma^d}}\right).$

Due to Lemma \ref{lem.aaba} and \eqref{eq.nw.proof}, we have with high probability that 

\begin{align}
    \left | \hat{u}_i(\boldx^\star) - \partial_i \log r(\boldx^\star) \right| \le& 
    \kappa \cdot \left( \frac{K}{\sqrt{n\sigma^d}} +   \frac{  \mathbbE_q[k_\sigma(\boldx, \boldx^\star) \| \boldx - \boldx^\star\| ]}{\mathbbE_q[k_\sigma(\boldx, \boldx^\star)]}\right) \notag \\
    \le & 
    \kappa \cdot \left( \frac{K}{\sqrt{n\sigma^d}} +   \frac{  \sigma \int q(\boldx^\star + \sigma\boldy )k(\boldy) \| \boldy\| \dy}{\int q(\boldx^\star + \sigma\boldy) k(\boldy) \dy}\right) 
        ~~~~~~~~~~~~~ \boldy := (\boldx - \boldx^\star)/\sigma
    \notag \\
    \le& \kappa \cdot \left( \frac{K}{\sqrt{n\sigma^d}}+ \sigma  C_k\right),  
    \notag 
\end{align}
where $K$ is constant. 

\end{proof}

\section{Visualization of Gradient  Estimation using Local Linear Fitting}

\begin{figure}
    \centering
        \begin{tikzpicture}
    \begin{axis}[
        xlabel={$\boldx$},
        ylabel={$g(\boldx)$},
        xmin=-1, xmax=.8,
        ymin=-.15, ymax=.5,
        grid=both,
        minor tick num=1,
        major grid style={line width=.2pt,draw=gray!50},
        minor grid style={line width=.2pt,draw=gray!25},
        width=8cm,
        height=5cm,
    ]
    
    \addplot[smooth, domain=-1:1, blue, line width=1.5pt] {x^2 + x^3};
    
    \pgfmathsetmacro\xZero{0.25}
    \pgfmathsetmacro\yZero{(\xZero)^2 + \xZero^3}
    \node[label={right:{$(\boldx^\star, \hat{g}(\boldx^\star))$}},circle,fill,inner sep=2pt,red] at (axis cs:\xZero,\yZero) {};
    
    \node[label={right:{$\hat{g}$}},inner sep=2pt,red] at (axis cs:0, -.1) {};
    
    \pgfmathsetmacro\derivAtXZero{2*\xZero + 3*(\xZero)^2}
    
    \addplot+[no marks, red, domain=0.0:0.5] { \derivAtXZero * (x - \xZero) + \yZero };
    
    \end{axis}
    \end{tikzpicture}
    \caption{Fit $g$ locally at $\boldx^\star$ using a linear function $\hat{g}$. Its ``slope'' is a natural estimator of $\nabla_\boldx g(\boldx^\star)$. }
    \label{fig:local-linear}
\end{figure}
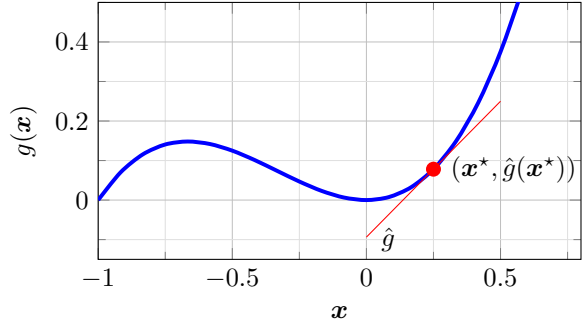

\section{Variational Objective for Estimating $h \circ r$}
\label{sec.var.obj.proof}
\begin{proposition}\label{def.hrx.estimator}
    The maximum in \eqref{eq.fdiv.variational} is attained if and only if $d = h \circ r$.
\end{proposition}
\begin{proof}
    Due to the maximizing argument, the maximum of \eqref{eq.fdiv.variational} is attained if and only if $d = \psi'$. The definition of mirror divergence indicates that $\psi' = r\phi'(r) - \phi(r)$. Theorem \ref{thm.gf} states that $h \circ r = r\phi'(r) - \phi(r) = \psi'$, thus the maximum is attained if and only if $d = h \circ r$. 
\end{proof}

\section{Proof of Theorem \ref{them.unbiasedness}, Estimation Error Bound of $\boldw(\boldx^\star)$}
\label{sec.proof.bias}

\begin{proof}
In this section, to simplify notations, we denote $\boldtheta$ as the parameter vector that combines both $\boldw$ and $b$, i.e., 
$\boldtheta: = [\boldw, b]^\top$.
Specifically, we define \[\boldtheta^* = \left[{\boldw^*}^\top, b^*\right]^\top := \left[ \nabla^\top (h \circ r)(\boldx^\star), h(r(\boldx^\star)) - \langle \nabla (h\circ r)(\boldx^\star), \boldx^\star \rangle  \right]^\top.\] 
Let us denote the negative objective function in \eqref{eq.kernelize.hrestimator} as $\ell(\boldtheta)$ 
and consider a constrained optimization problem: 
\begin{align}
\label{eq.obj.constrained}
    \min_\boldtheta \ell(\boldtheta) \text{ subject to:} \|\boldtheta - \boldtheta^*\|^2 \le \min(W, B)^2. 
\end{align}
This convex optimization has a Lagrangian $\ell(\boldtheta) + \lambda ( \|\boldtheta - \boldtheta^*\|^2 - \min(W, B)^2)$, where $\lambda \ge 0$ is the Lagrangian multiplier. 
According to KKT condition, the optimal solution $\boldtheta_0$ of \eqref{eq.obj.constrained} satisfies $\nabla \ell(\boldtheta_0) + 2\lambda (\boldtheta_0 - \boldtheta^*)= \boldzero$. 

We apply mean value theorem to $g(\boldtheta) := \langle \boldtheta_0 - \boldtheta^*,  \nabla \ell(\boldtheta) + 2\lambda (\boldtheta - \boldtheta^*)\rangle$: 
\begin{align*}
    \underbrace{\langle \boldtheta_0 - \boldtheta^*,  \nabla \ell(\boldtheta_0) + 2\lambda  (\boldtheta_0 - \boldtheta^*)}_{ g(\boldtheta_0) = 0, \text{ KKT condition}} \rangle 
    = \underbrace{\langle \boldtheta_0 - \boldtheta^*,  \nabla \ell(\boldtheta^*)\rangle}_{g(\boldtheta^*)} + \langle \underbrace{(\boldtheta_0 - \boldtheta^*)^\top \left[\nabla^2 \ell(\bar{\boldtheta})+ 2\lambda \boldI \right]}_{\nabla g(\bar{\boldtheta})}, (\boldtheta_0 - \boldtheta^*)\rangle,
\end{align*}
where $\bar{\boldtheta}$ is a point between $\boldtheta^*$ and $\boldtheta_0$ in an elementwise fashion. 
Since $\bar{\boldtheta}$ is in a hyper cube with $\boldtheta_0$ and $\boldtheta^*$ as opposite corners, it is in the constrain set of \eqref{eq.obj.constrained}. 
Let us rearrange terms:
\begin{align*}
    -\langle \boldtheta_0 - \boldtheta^*,  \nabla \ell(\boldtheta^*)\rangle = \langle \boldtheta_0 - \boldtheta^*,  \left[\nabla^2 \ell(\bar{\boldtheta}) + 2\lambda \boldI \right] (\boldtheta_0 - \boldtheta^*)\rangle \ge \langle \boldtheta_0 - \boldtheta^*,  \nabla^2 \ell(\bar{\boldtheta}) (\boldtheta_0 - \boldtheta^*)\rangle, 
\end{align*}
For all $\boldtheta$ in the constraint set, $\|\boldw\| = \|\boldw - \boldw^* + \boldw^*\| \le \|\boldw - \boldw^* \| + \|\boldw^*\| \le \|\boldtheta - \boldtheta^* \| + \|\boldw^*\| \le 2W$ and $|b| = |b - b^* + b^*| \le |b - b^* | + |b^*| \le \|\boldtheta - \boldtheta^*\| + |b^*| \le 2B$.
Under our assumption \eqref{eq.eigenvalu.condition}, the lowest eigenvalue of $\nabla^2\ell(\boldtheta)$ for all $\boldtheta$ in the constrain set of \eqref{eq.obj.constrained} is always lower bounded by $\sigma^d \Lambda_\mathrm{min}$. Therefore, 
\begin{align*}
    -\langle \boldtheta_0 - \boldtheta^*,  \nabla \ell(\boldtheta^*)\rangle \ge \sigma^d \cdot \Lambda_\mathrm{min}\| \boldtheta_0 - \boldtheta^*\|^2. 
\end{align*}
Using Cauchy–Schwarz inequality
\begin{align*}
    \| \boldtheta_0 - \boldtheta^* \|  \|\nabla \ell(\boldtheta^*)\| \ge \sigma^d \cdot \Lambda_\mathrm{min}\| \boldtheta_0 - \boldtheta^*\|^2. 
\end{align*}
Assume $\| \boldtheta_0 - \boldtheta^* \|$ is not zero (if it is, our estimator is already consistent).  
\begin{align}
    \label{eq.upperbound.locallinear.kl}
     \| \boldtheta_0 - \boldtheta^*\|  \le & \frac{1}{\sigma^d\Lambda_\mathrm{min}}\|\nabla \ell(\boldtheta^*)\| \notag \\
    \le &\frac{1}{\sigma^d\Lambda_\mathrm{min}}\|\nabla \ell(\boldtheta^*) - \mathbbE \nabla \ell(\boldtheta^*) + \mathbbE \nabla \ell(\boldtheta^*) \| \notag \\
    \le & \frac{1}{\sigma^d\Lambda_\mathrm{min}}\left(\|\nabla \ell(\boldtheta^*)  - \mathbbE \nabla \ell(\boldtheta^*)\| + \|\mathbbE \nabla \ell(\boldtheta^*) \| \right)
\end{align}
Since $\mathrm{tr}\left[{\mathrm{Cov}_p\left[\frac{1}{\sigma^d}{k(\boldx, \boldx^\star)\cdot \tilde{\boldx}}\right]}\right] = O(\frac{1}{{\sigma^d}})$ and $\mathrm{tr}\left[{\mathrm{Cov}_q\left[\frac{1}{\sigma^d}{k(\boldx, \boldx^\star)\cdot \psi_\mathrm{con}(
\langle \boldw^*, \boldx\rangle + b^*)\tilde{\boldx}}\right]}\right] = O(\frac{1}{{\sigma^d}})$ by assumption, due to multi-dimensional version of Chebyshev's inequality,  $\sigma^d\|\frac{1}{\sigma^d}\left(\nabla \ell(\boldtheta^*) - \mathbbE \nabla \ell(\boldx^*) \right) \| = \sigma^d O_p(\frac{1}{\sqrt{n\sigma^d}})$. Therefore, 
\begin{align*}
    \| \boldtheta_0 - \boldtheta^*\| \le \frac{1}{\sigma^d\Lambda_\mathrm{min}}\left( \sigma^d \cdot \frac{K}{\sqrt{n\sigma^d}} + \|\mathbbE \nabla \ell(\boldtheta^*) \| \right),
\end{align*}
with high probability and $K$ is a constant.

Now we proceed to bound $\|\mathbbE \nabla \ell (\boldtheta^*)\|$. 
\begin{lemma}
\label{lem.gradl.bound}
    $\|\mathbbE \nabla \ell(\boldtheta^*)\| \le \mathbbE_q\left[k_\sigma(\boldx, \boldx^\star)\frac{1}{2}  \left \|\boldx- \boldx^\star \right \|^2 
\|\tilde{\boldx}\|\right]\kappa \cdot C_{\psi_\mathrm{con}''}$  
\end{lemma}
\begin{proof}
The expression of $\mathbbE\nabla \ell(\boldtheta^*)$ is:
    \begin{align}
    \label{eq.nabla_ell_kliep}
        \mathbbE \nabla \ell (\boldtheta^*) := - \mathbbE_p\left[{k_\sigma(\boldx, \boldx^\star)} \cdot \tilde{\boldx}\right] + \mathbbE_q\left[k_\sigma(\boldx, \boldx^\star) \psi'_\mathrm{con} (\langle \boldw^*, \boldx \rangle + b^*) \cdot  \tilde{\boldx} \right]. 
    \end{align}
Due to Taylor's theorem, $\langle \boldw^*, \boldx \rangle + b^* = h(r(\boldx))  - \frac{1}{2} \left(\boldx- \boldx^\star\right)^\top \nabla^2 (h \circ r)(\bar{\boldx}) \left(\boldx- \boldx^\star\right)$  
where $\bar{\boldx}$ is a point in between $\boldx$ and $\boldx^\star$ in an elementwise fashion. 
Thus, applying the mean value theorem on $\psi'_\mathrm{con}$,
\begin{align*}
    \psi_\mathrm{con}'(\langle \boldw^*, \boldx \rangle + b^*) &= 
\psi_\mathrm{con}'\left[h(r(\boldx))  - \frac{1}{2} \left(\boldx- \boldx^\star\right)^\top \nabla^2 (h \circ r)(\bar{\boldx}) \left(\boldx- \boldx^\star\right)\right] \\
&= \psi_\mathrm{con}'\left[h(r(\boldx)) 
\right] - \frac{1}{2} \left(\boldx- \boldx^\star\right)^\top \nabla^2 (h \circ r)(\bar{\boldx}) \left(\boldx- \boldx^\star\right) \psi_\mathrm{con}''(y),
\end{align*}
where $y$ is a scalar in between $h(r(\boldx))$ and $h(r(\boldx)) - \frac{1}{2} \left(\boldx- \boldx^\star\right)^\top \nabla^2 (h \circ r)(\bar{\boldx}) \left(\boldx- \boldx^\star\right)$ or equivalently, in between $h(r(\boldx))$ and $\langle \boldw^*, \boldx \rangle + b^*$. 

Theorem \ref{thm.gf} states that $h(r) = r\phi' + \phi$,
then, by the definition of the mirror divergence, $\psi' = h(r)$. Moreover, due to the maximizing argument, $\psi'_\mathrm{con}$ is the input argument of $\psi$ (i.e., $r$)
and $\psi'$ is the input argument of $\psi_\mathrm{con}$. 
Thus, $\psi_\mathrm{con}'(h(r)) = \psi'_\mathrm{con} (\psi'(r)) = r$.
Let us write 
\begin{align*}
\mathbbE \nabla \ell(\boldtheta^*) &= -\mathbbE_p[k_\sigma(\boldx, \boldx^\star) \tilde{\boldx}] + \mathbbE_q[k_\sigma(\boldx, \boldx^\star) \underbrace{\psi_\mathrm{con}'(h(r(\boldx)))}_{r}\tilde{\boldx}] 
\\ &~~ - \mathbbE_q\left[k_\sigma(\boldx, \boldx^\star)\frac{1}{2} \left(\boldx- \boldx^\star\right)^\top \nabla^2 (h \circ r)(\bar{\boldx}) \left(\boldx- \boldx^\star\right)\psi_\mathrm{con}''(y)\tilde{\boldx}\right]\\
&= - \mathbbE_q\left[k_\sigma(\boldx, \boldx^\star)\frac{1}{2} \left(\boldx- \boldx^\star\right)^\top \nabla^2 (h \circ r)(\bar{\boldx}) \left(\boldx- \boldx^\star\right)\psi_\mathrm{con}''(y)\tilde{\boldx}\right], 
\end{align*}
We can derive a bound for $\|\mathbbE \nabla \ell(\boldtheta^*)\|$:
\begin{align*}
\|\mathbbE \nabla \ell(\boldtheta^*)\| &\le \mathbbE_q\left[k_\sigma(\boldx, \boldx^\star)\cdot \frac{1}{2} \left |\left(\boldx- \boldx^\star\right)^\top \nabla^2 (h \circ r)(\bar{\boldx}) \left(\boldx- \boldx^\star\right)\right| \cdot |\psi_\mathrm{con}''(y)| \cdot \|\tilde{\boldx}\|\right] \\
&\le \mathbbE_q\left[k_\sigma(\boldx, \boldx^\star)\frac{1}{2}  \left \|\boldx- \boldx^\star \right \|^2 
\left \|\nabla^2 (h \circ r)(\bar{\boldx}) \right\|
\|\tilde{\boldx}\|\right]C_{\psi_\mathrm{con}''}\\
&\le \mathbbE_q\left[k_\sigma(\boldx, \boldx^\star)\frac{1}{2}  \left \|\boldx- \boldx^\star \right \|^2 
\|\tilde{\boldx}\|\right]\kappa \cdot C_{\psi_\mathrm{con}''}\\
& \le  \sigma^{d+2} \cdot \kappa \cdot C_{\psi_\mathrm{con}''} \frac{1}{2} \int q(\sigma\boldy + \boldx^\star) \cdot k(\boldy) \cdot \left \|\boldy \right\|^2 \cdot \left \|[\sigma \boldy + {\boldx^\star}, 1]\right \|  \dy 
 \le  \sigma^{d+2} \cdot \kappa \cdot C_{\psi_\mathrm{con}''} C_k
\end{align*}
\end{proof}
Finally, due to Lemma \ref{lem.gradl.bound} and Assumption \ref{ass.nw2}, 
we can see that with high probability 
\[\|\boldtheta^* - \boldtheta_0\| \le  \frac{ \frac{K\sigma^d}{\sqrt{n\sigma^d}} +  \kappa \cdot \sigma^{d+2} \cdot C_k \cdot C_{\psi_\mathrm{con}''}}{\sigma^d \Lambda_\mathrm{min}} \le \frac{\frac{K}{\sqrt{n\sigma^d}} + \kappa \cdot \sigma^{2} \cdot C_k \cdot C_{\psi_\mathrm{con}''}}{\Lambda_\mathrm{min}}. \]

Since $n\sigma^d \to \infty, \sigma \to 0$,  $
\|\boldtheta^* - \boldtheta_0\|\to 0$ with high probability. There always exists $\sigma_0$ and $N$, such that for all $\sigma < \sigma_0$ and $n > N$, $\min (W,B) >  \frac{\frac{K}{\sqrt{n\sigma^d}} + \kappa \cdot C_k \cdot C_{\psi_\mathrm{con}''}\cdot \sigma^2}{\Lambda_\mathrm{min}}$. When it happens,  $\boldtheta_0$ must be the interior of the constrain set of \eqref{eq.obj.constrained}. i.e., the constraints in \eqref{eq.obj.constrained} are not active. It implies $\boldtheta_0$ must be the stationary point of $\ell(\boldtheta)$ as long as $\sigma$ is sufficiently small and $n$ is sufficiently large. 

\end{proof}

\section{Proof of Corollary \ref{co.gradratio}}
\label{proof.col.ulsif}
\begin{proof}
Since \eqref{eq.obj.thetax0.ls} has a unconstrained quadratic objective, its maximizers are stationary points.

We can see that Assumption \ref{ass.curvature} holds. 
To apply Theorem \ref{them.unbiasedness}, we still need to show that Assumption \ref{ass.psi.twice} holds and 
$W$ and $B$ exist.
In this case, $\psi_\mathrm{con}(d) = d^2/2 + d$, 
so $\psi_\mathrm{con}'' = 1$. Thus
Assumption \ref{ass.psi.twice} holds automatically for every $C_{\psi_\mathrm{con}''} 
\ge 1$.
Additionally,
\[
\nabla_{[\boldw, b]}^2 \psi_\mathrm{con}(\langle \boldw, \boldx \rangle + b) = \hatE_q[k_\sigma(\boldx, \boldx^\star)\psi''_\mathrm{con}(\langle \boldw, \boldx\rangle + b)\tilde{\boldx}\tilde{\boldx}^\top] = \hatE_q[k_\sigma(\boldx, \boldx^\star)\tilde{\boldx}\tilde{\boldx}^\top].
\] 
Therefore, \eqref{eq.cond.hessian}  implies the minimum eigenvalue assumption \eqref{eq.eigenvalu.condition} holds for every $W >0, B>0$. 
Thus, we can choose any $W > 0$ and $B>0$ that satisfies \eqref{eq.upperbounding.condition}.
Noticing that $h(r(\boldx)) = r(\boldx)$, applying Theorem \ref{them.unbiasedness} gives the desired result. 
\end{proof}

\section{Proof of Corollary \ref{co.gradlogratio}}
\label{proof.co.logratio}
\begin{proof}
    Assumption \ref{ass.curvature}, \ref{ass.nw2} and \eqref{eq.upperbounding.condition} are already satisfied. 
    Let us verify the eigenvalue condition \eqref{eq.eigenvalu.condition}.  In this case, $\psi_\mathrm{con}''(d) = \exp(d-1)$. Thus $\exp(\langle \boldw, \boldx\rangle + b -1) \le \exp(2WC_\mathcal{X}+2B-1) < \infty$ for $\|\boldw\|\le 2W, |b| \le 2B$. Moreover, because $\tilde{\boldx}\tilde{\boldx}^\top$ is positive semi-definite, due to \eqref{eq.cond.eigen}, 
    \begin{align*}
\lambda_\mathrm{min}\left[\hatE_q[k_\sigma(\boldx, \boldx^\star ) \exp(\langle\boldw, \boldx\rangle + b -1 )\tilde{\boldx}\tilde{\boldx}^\top]\right] \ge &\sigma^2 \cdot \lambda_\mathrm{min} \left[\hatE_q[k_\sigma(\boldx, \boldx^\star )\tilde{\boldx} \tilde{\boldx}^\top]\right]  \cdot \exp(-2WC_\mathcal{X} - 2B-1) \\
>&\sigma^2 \cdot \Lambda_\leftarrow  \cdot \exp(-2WC_\mathcal{X} - 2B -1) >0,
    \end{align*}
    for all $\boldw, b$ that $ \|\boldw\|\le 2W, |b| \le 2B$.
So \eqref{eq.eigenvalu.condition} holds. 
Finally, let us verify Assumption \ref{ass.psi.twice}. 
Since $\psi_{\mathrm{con}''}$ is a strictly monotone increasing function, $\sup_{a \in [0,1]} \psi_{\mathrm{con}''}(ax_0 + (1-a)y_0)$ is obtained either at $x_0$ or $y_0$. 
We only need to verify that $\psi_{\mathrm{con}''}(h(r(\boldx)))$ and $\psi_{\mathrm{con}''}(\langle \boldw^*, \boldx\rangle + b^*)$ are both bounded for all $\boldx$.
Both $h(r(\boldx))$ and $\langle \boldw^*, \boldx\rangle + b^*$ can be bounded using our assumptions. Thus, for a \[C_{\psi_\mathrm{con}''} = \exp(W \cdot C_\mathcal{X} + B -1) \vee  \exp(C_{\log r} - 1) \]  Assumption \ref{ass.psi.twice} holds. 
     Applying Theorem \ref{them.unbiasedness} completes the proof. 
\end{proof}

\section{$\|\nabla^2 r(\boldx)\|$ for Different $p$ and $q$}
\begin{figure*}[t]
    \centering
    \subfigure[$p,q$ with different means. The further $p,q$ are apart, the larger the $\sup_{\boldx}\|\nabla^2 r(\boldx)\|$ is.]{
        \includegraphics[width = .9\textwidth]{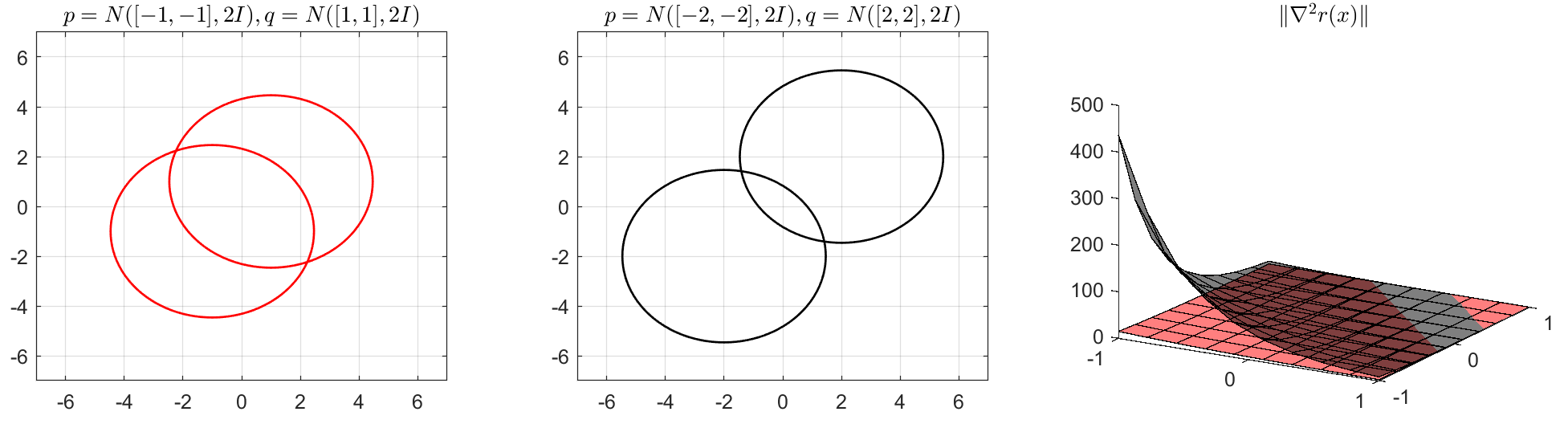}
    }
    \subfigure[$p,q$ with different variances. The less overlapped $p,q$ are, the larger the $\sup_{\boldx} \|\nabla^2 r(\boldx)\|$ is.]{
        \includegraphics[width = .9\textwidth]{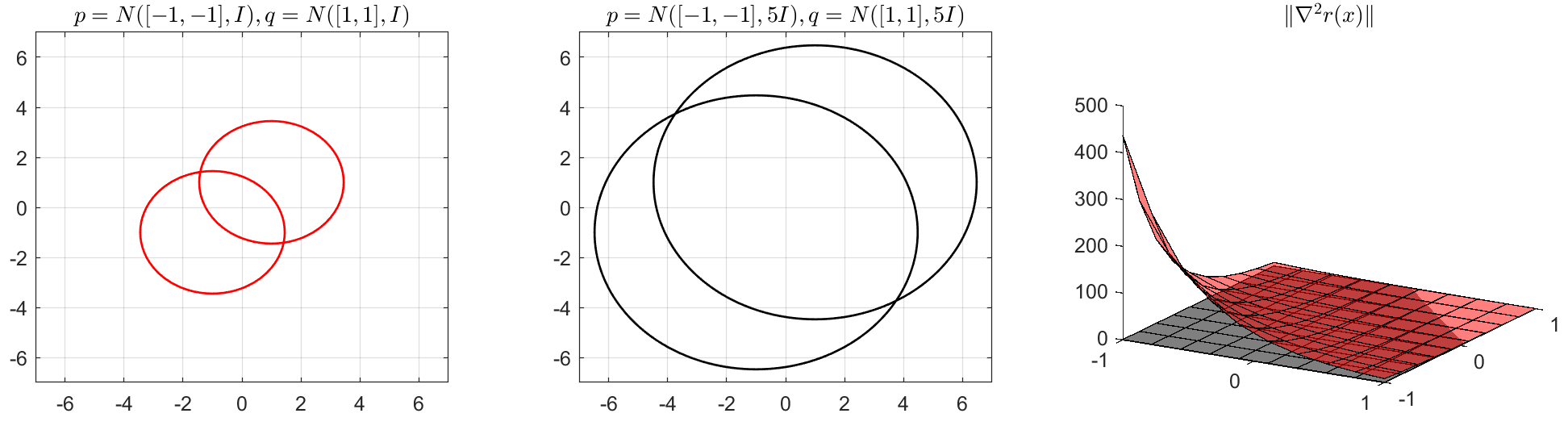}
    }
    \caption{Visualizing $\nabla^2 (h\circ r)(\boldx)$ for $\mathrm{D}_{\phi} = \mathrm{KL}[p,q]$ with two different settings of $p$ and $q$. }
    \label{fig:enter-label.appendix}
\end{figure*}
See Figure \ref{fig:enter-label.appendix}.

\section{Finite-sample Objectives}
\subsection{Practical Implementation}

In our experiments, we observe that \eqref{eq.kernelize.hrestimator} can be efficiently minimized by using gradient descent with adaptive learning rate schemes, e.g., Adam \citep{kingma2014adam}. 
One computational advantage of the local linear model is that the computation for each $\boldx^\star$ is independent from the others. This property allows us to parallelize the optimization. 
Even using a single CPU/GPU, we can easily write highly vectorized code to 
compute the gradient of \eqref{eq.kernelize.hrestimator} with respect to $\boldw$ and $b$ for a large particle set $·\{\boldx_i^\star\}_{i=1}^n$.

Suppose 
$\boldX_p \in \mathbbR^{n_p \times d}$ and $\boldX_q \in \mathbbR^{n_q \times d}$ are the matrices whose rows are $\boldx_p^{(i)}$ and $\boldx_q^{(i)}$ respectively.
$\boldK_p \in \mathbbR^{n \times n_p}$, $\boldK_q \in \mathbbR^{n \times n_q}$ are the kernel matrices between $\{\boldx^\star\}$ and $\boldX_p$, $\{\boldx^\star \}$ and $\boldX_q$ respectively.
$\boldW \in \mathbbR^{n \times d}$ and $\boldb \in \mathbbR^n$ are the parameters whose rows are $\boldw(\boldx^\star)$ and $b(\boldx^\star)$ respectively. Then the gradient of \eqref{eq.kernelize.hrestimator} with respect to $\boldW$, $\boldb$ can be expressed as 
\begin{align*}
 \boldK_p \boldX_p /n_p 
- \boldK_q \odot \psi_\text{con}'( \boldW\boldX_q^\top + \boldb)  \boldX_q  /n_q, 
\\
 \boldK_p \boldone_{n_p} /n_p - \boldK_q \odot \psi_\text{con}'( \boldW\boldX_q^\top + \boldb) \boldone_{n_q} /n_q, 
\end{align*}
where $\boldone_{n_p}$ and $\boldone_{n_q}$ are the vectors of ones with length $n_p$ and $n_q$ respectively. 
$\psi_\mathrm{con}'$ is evaluated element-wise.
$\odot$ is the element-wise product and the vector $\boldb$ is broadcast to a matrix with $n_q$ columns.

\section{Experiment Details in Section \ref{sec.exp}}
\label{sec.expsettings}

\subsection{Model Selection}
\label{sec.mosel}
Let
$\mathcal{D}_p:= \left\{\boldx_p^{(i)}\right \}_{i=1}^{n_p}, \mathcal{D}_q := \left \{\boldx_q^{(i)}\right\}_{i=1}^{n_q}$ be training sets from $p$ and $q$ respectively and 
$\tilde{\mathcal{D}}_{p} = \left\{\tilde{\boldx}_{p}^{(i)}\right\}_{i = 1}^{\tilde{n}_{p}}$ and $\tilde{\mathcal{D}}_{q} = \left\{\tilde{\boldx}_{q}^{(i)}\right\}_{i = 1}^{\tilde{n}_{q}}$ be testing sets. 
We can fit a local linear model \textit{at each testing point using the training sets}, i.e., 
\begin{align*}
    \left(\hat{\boldw}_\sigma(\tilde{\boldx}), \hat{b}_\sigma(\tilde{\boldx})\right) &:= \argmax_{\boldw \in \mathbbR^d, b \in \mathbbR} \underbrace{\ell\left(\boldw, b; \tilde{\boldx}, \mathcal{D}_{p},\mathcal{D}_{q}\right)}_{\text{tranining objective}}.    
\end{align*}
The dependency on $\sigma$ comes from the smoothing kernel in the training objective. 
We can tune $\sigma$ by evaluating the variational lower bound \eqref{eq.fdiv.variational} approximated using testing samples: 
\begin{align}
    \label{eq.testingloss}
    \tilde{\ell}\left(\sigma; \tilde{\mathcal{D}}_{p}, \tilde{\mathcal{D}}_{q}\right)
    &:= \underbrace{  \tilde{\mathbbE}_p
    \left[
    \hat{d}_\sigma\left(\boldx\right)
    \right] -  \tilde{\mathbbE}_q
    \left[\psi_\text{con}(\hat{d}_\sigma\left(\boldx\right))\right]}_{\text{testing criterion}}, 
\end{align}
where $\tilde{\mathbbE}_p [\hat{d}(\boldx)] := \frac{1}{\tilde{n}_p} \sum_{i=1}^{\tilde{n}_p} \hat{d}(\tilde{\boldx}_p^{(i)}), 
$ i.e., the sample average over the testing points. 
$\hat{d}_\sigma(\tilde{\boldx}) := \langle \hat{\boldw}_\sigma\left(\tilde{\boldx}\right), \tilde{\boldx}\rangle + \hat{b}_\sigma\left(\tilde{\boldx}\right)$, is the interpolation of $d(\tilde{\boldx})$ using training samples. 
The best choice of $\sigma$ should maximize the above testing criterion. 

In our experiments, we construct training and testing sets using cross validation and choose a list of candidate $\sigma$ for the model selection. This procedure is parallel to selecting $k$ in $k$-nearest neighbors to minimize the testing error. In our case, \eqref{eq.testingloss} is the ``negative testing error''.

\subsection{Missing Data Imputation}
\label{sec.missing.exp}
Before running the experiments, we first pre-process data in the following way:
\begin{enumerate}
    \item Suppose $\boldX_{\text{true}}$ is the original data matrix, i.e. without missing values. We introduce missingness to $\boldX_{\text{true}}$, and call the matrix with missing values $\tilde{\boldX}$, following MCAR paradigm. Denote the corresponding mask matrix as $\boldM$, where $m_{j}^{(i)}=0$ if ${\tilde{x}}^{(i)}_{j}$ is missing, and $m_{j}^{(i)}=1$ otherwise.
    \item Calculate column-wise mean $\bar{\tilde{\boldx}}$ and standard deviation $\tilde{\bolds}$ (excluding missing values)  of 
$\tilde{\boldX}$.
    \item Standardize $\tilde{\boldX}$ by taking $\tilde{\boldX} = \frac{\tilde{\boldX} - \bar{\tilde{\boldx}}}{\tilde{\bolds}}$, where the vectors $\bar{\tilde{\boldx}}$ and $\tilde{\bolds}$ are broadcasted to the same dimensions as the matrix $\tilde{\boldX}$. Note that the division here is element-wise.
\end{enumerate}

Denote $\boldX^{t}$ as the imputed data of $\tilde{\boldX}$ at iteration $i$, where $\boldX^0=\tilde{\boldX} \odot \boldM + \mathbf{Z} \odot (1-\boldM)$, and $\mathbf{Z} \sim \mathcal{N}(\mathbf{0},\mathrm{diag}(\mathbf{1}))$.

We performed two experiments on both toy data (``S''-shape) and real world data (UCI Breast Cancer \footnote{Available at https://archive.ics.uci.edu/ml/machine-learning-databases/breast-cancer-wisconsin/wdbc.data.} data).

Let $N_{\text{WGF}}$ be the number of iterations WGF is performed. In each iteration, let $N_{\text{GradEst}}$ be the number gradient descent steps for gradient estimation. In the missing data experiments, we set the hyperparameters to be:
\begin{itemize}
    \item ``S''-shape data: $T_{\text{WGF}} = 100$, $T_{\text{GradEst}} = 2000$, $\sigma$ is chosen by model selection described in Section \ref{sec.mosel}.
    \item UCI Breast Cancer data: $T_{\text{WGF}} = 1000$, $T_{\text{GradEst}} = 100$, $\sigma = \text{median}  (\sqrt{\frac{\text{pairwise distance of} \boldX}{2}})$ .
\end{itemize}

\begin{figure}[t]
\centering
    \includegraphics[width=.4\textwidth]{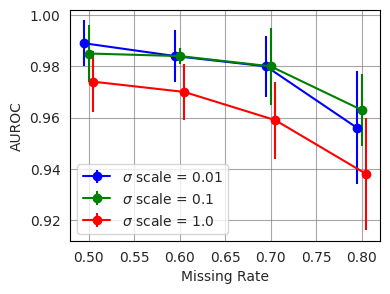}
    \caption{AUROC of a linear SVM classifier on the imputed Breast Cancer dataset, with various scales of $\sigma$ from 0.01 to 1. }
    \label{fig.imp.sigma}
\end{figure}

Across experiments, we observe that our method is robust against the choice of the bandwidth $\sigma$. We demonstrate this via a new experiment that tests the imputation performance with different selections of $\sigma$. Figure \ref{fig.imp.sigma} demonstrates this experiment on the UCI Breast Cancer data. We can see that the imputation performance is still comparable with the original results (and baseline methods) when the scale of $\sigma$ varies from 0.01 to 1.

\subsection{Wasserstein Gradient Flow}
In this experiment, we first expand MNIST digits into $32 \times 32$ pictures then adds a small random noise $\epsilon \sim \mathcal{N}(\boldzero, 0.001^2 \boldI)$ to each picture so that computing the sample mean and covariance will not cause numerical issues. For both forward and backward KL WGF, we use a kernel bandwidth that equals to $1/5$ of the pairwise distances in the particle dataset, as it is too computationally expensive to perform cross validation at each iteration. After each update, we clip pixel values so that they are in between $[0, 1]$. It is done using pyTorch \verb|torch.clamp| function.

To reduce computational cost, at each iteration, we randomly select 4000 samples from the original dataset and 4000 particles from the particle set. We use these samples to estimate 
the WGF updates.

\section{Discussion:Kernel Density Gradient Estimation}
\label{sec.kde.app}

The Kernel Density Estimator (KDE) of $p$ is \begin{align*}
    \hat{p}(\boldy) := \frac{1}{n_p} \sum_{i=1}^{n_p} k(\boldx_p^{(i)}, \boldy) /  Z, 
\end{align*}
where $Z$ is a normalization constant to ensure that $\int \hat{p}(\boldx) \dx = 1$. Thus, 
\[
    \nabla_\boldy \log \hat{p} (\boldy):= \frac{1}{n_p} \sum_{i=1}^{n_p} \nabla_\boldy k(\boldx_p^{(i)}, \boldy) /  \frac{1}{n_p} \sum_{i=1}^{n_p} k(\boldx_p^{(i)}, \boldy). 
\]
The normalizing constant $Z$ is cancelled. 

\section{Discussion: Why Score Matching does not Work on Log Ratio Gradient Estimation}
\label{sec.why.sm.doesnt.work}

For Score Matching (SM), the estimator of $\hat{p} := \argmin_f \int p \|\nabla \log p - \nabla \log f\|^2 \dx$, where the objective function is commonly refered to as \emph{Fisher Divergence}. To use SM in practice, the objective function is further broken down to 
\begin{align}
    \label{eq.sm.obj}
    \int p \|\nabla \log p - \nabla \log f\|^2 \dx =  \int p \|\nabla \log f\|^2 \dx + 2 \sum_{i=1}^d \int p \partial_i^2 \log f \dx + C, 
\end{align}
where we used the dimension-wise integration by parts and $C$ is a constant. 

Since our target is to estimate $\nabla \log p$, we can directly model $\nabla \log p$ as $\boldg: \mathbbR^d \to \mathbbR^d$. The objective becomes 
\begin{align*}
    \int p \|\boldg\|^2 \dx + 2 \sum_{i=1}^d \int p \partial_i g_i \dx + C. 
\end{align*}

SM can be used to estimate $\nabla \log p$. One might assume that SM can also be used for estimating $\nabla \log r$, where $r := \frac{p}{q}$. Let us replace $\nabla \log p$ with $\nabla \log r$ in \eqref{eq.sm.obj}, 
\begin{align}
    \label{eq.sm.obj.ratio}
    \int p \|\nabla \log r - \nabla \log f\|^2 \dx &=  \int p \|\nabla \log f\|^2 \dx - 2 \sum_{i=1}^d \int p \partial_i \log r \partial_i \log f \dx  + C \notag \\
    &= \int p \|\nabla \log f\|^2 \dx + 2 \sum_{i=1}^d \int q \partial_i  r \partial_i \log f \dx  + C \notag \\
    &= \int p \|\nabla \log f\|^2 \dx + 2 \sum_{i=1}^d  \int r \cdot \partial_i (q \partial_i  \log f) \dx  + C, \\
    &= \int p \|\nabla \log f\|^2 \dx + 2 \sum_{i=1}^d  \int r \cdot q \cdot \partial^2_i  \log f \dx + 2 \sum_{i=1}^d  \int r \cdot \partial_i q \cdot \partial_i  \log f \dx  + C \label{eq.sm.obj.ratio.2}
\end{align}
and to get \eqref{eq.sm.obj.ratio} we applied integration by parts, where we assumed $q\cdot r \cdot \partial_i \log f \to 0$ as $x_i \to \infty$. In \eqref{eq.sm.obj.ratio.2}, the third term is not tractable due to the lack of information about $r$ and $q$. Changing the objective to $\int p \|\nabla \log r - \nabla \log f\|^2 \dx$ would also yield an intractable objective for a similar reason.

\section{Discussion: Gradient Flow Estimation in Feature Space}
\label{sec.wgf.fea}

One of the issues of local estimation is the \emph{curse of dimensionality}: Local approximation does not work well in high dimensional spaces. 
However, since the $f$-divergence gradient flow is always associated with the density ratio function, we can utilize special structures in  density ratio functions to estimate $\nabla (h\circ r)(\boldx^\star)$ more effectively.

\subsection{Density Ratio Preserving Map}

Let \(\boldsymbol{s}(\boldsymbol{x})\) be a measurable function, where \(\boldsymbol{s}: \mathbb{R}^d \rightarrow \mathbb{R}^m, m \le d\).
Consider two random variables, \(X_p\) and \(X_q\), each associated with probability density functions \(p\) and \(q\), respectively.
Define \(p^\circ\) and \(q^\circ\) as the probability density functions of the random variables \(S_p := \boldsymbol{s}(X_p)\) and \(S_q := \boldsymbol{s}(X_q)\).

\begin{definition}
\label{def.drpreservingmap}
    $\bolds(\boldx)$ is a density ratio preserving map if and only if it
satisfies the following equality \[r(\boldx) = r^\circ(\bolds(\boldx)),   \text{where } r^\circ(\bolds(\boldx)) := \frac{p^\circ(\bolds(\boldx))}{q^\circ(\bolds(\boldx))}, ~~~ \forall \boldx \in \mathcal{X}.\] 
\end{definition}
We can leverage the density ratio preserving map to reduce the dimensionality of gradient flow estimation. 
Suppose $\bolds(\boldx)$ is a known density ratio preserving map. 
Define $\boldz := \bolds(\boldx)$ and $\boldz^\star := \bolds(\boldx^\star)$.
We can see that
\begin{align}
\label{eq.dimred}
    \underbrace{\nabla (h\circ r)(\boldx^\star)}_{\mathbbR^d \mapsto \mathbbR^d} = \nabla 
    (h \circ r^\circ)
    (\bolds(\boldx^\star)) = \underbrace{\nabla^\top \bolds(\boldx^\star)}_{\mathbbR^d \mapsto \mathbbR^{d \times m},\text{ known}} \underbrace{\nabla (h \circ r^\circ) (\boldz^\star)}_{\mathbbR^d \mapsto \mathbbR^{\color{red}m}}.
\end{align}
If we can evaluate $\nabla \bolds(\boldx^\star)$, we only need to estimate  
an $m$-dimensional gradient  $\nabla (h\circ r^\circ)(\boldz^\star)$,
which is potentially easier than estimating the original $d$-dimensional gradient $\nabla (h\circ r)(\boldx^\star)$ using a local linear model. 

While Definition \ref{def.drpreservingmap} might suggest that 
$\bolds$ is a very specific function, the requirement for 
$\bolds$ to preserve the density ratio is quite straightforward. Specifically, 
$\bolds$
must be \textit{sufficient in expressing the density ratio function.} This requirement is formalized in the following proposition:

\begin{proposition}\label{lem.dim.red}
Consider a function \(\boldsymbol{s}: \mathbb{R}^d \to \mathbb{R}^m\). If there exists a function \(g: \mathbb{R}^m \to \mathbb{R}_+\) such that  \(r(\boldsymbol{x}) =  g(\boldsymbol{s}(\boldsymbol{x})), \forall \boldsymbol{x} \in \mathcal{X}\) holds, then \(\boldsymbol{s}\) is a density ratio preserving map. Additionally, it follows that \( g \circ \bolds = r = r^\circ \circ \bolds  \implies g = r^\circ.\)
\end{proposition}

The proof can be found in Section~\ref{sec.density.ratio.pres.proof}.

Proposition \ref{lem.dim.red} implies that we can identify the density ratio preserving map $\bolds$ by simply learning the ratio function $r$ and using the trained feature transform function as $\bolds$.
For instance, in the context of a neural network used to estimate \(r\), \(\boldsymbol{s}\) could correspond to the functions represented by the penultimate layer of the network. 
After identifying $\bolds$,
we can simply translate a  high dimensional gradient flow estimation into a low dimensional problem according to \eqref{eq.dimred}.

In practice, we find this method works well. However, this approach still requires us to estimate a high dimensional density ratio function $r$ to obtain a feature map $\bolds$.

In the next section, we propose an algorithm of learning $\bolds$ from data without estimating a high dimensional density ratio function.  

\subsection{Finding Density Ratio Preserving Map }
\begin{algorithm}[t]
\caption{Searching for a Density Ratio Preserving Map $\bolds$ }
\begin{algorithmic}[1]
    \STATE \textbf{Inputs:} $\mathcal{D}_p$, $\mathcal{D}_q$ and an initial guess of $\hat{\bolds}$. 
    \WHILE{$\hat{\bolds}$ not converged}
        \FOR{each $\boldx \in \mathcal{D}_p \cup \mathcal{D}_q$}
            \STATE $\boldz := \hat{\bolds}(\boldx)$
            \STATE  $\left(\hat{\boldw}(\boldz), \hat{b}(\boldz)\right) := \argmin_{\boldw\in \mathbbR^{\color{red}m}, b\in \mathbbR} \ell(\boldw, b; \boldz, \mathcal{D}_{{\color{red}{p^\circ}}},\mathcal{D}_{\color{red}{q^\circ}})\mid_{\color{red}\bolds = \hat{\bolds}}$ 
            \STATE $\hat{d}(\boldz) :=
    \langle  \hat{\boldw}\left(\boldz\right), \boldz\rangle + \hat{b}\left(\boldz\right)$
        \ENDFOR
        \STATE $\hat{\bolds} := \argmax_{\bolds\in \mathbb{S}} \widehat{\mathbbE}_p\left[\hat{d}\left(\bolds\left(\boldx\right)\right)\right]
    - \widehat{\mathbbE}_q\left[
    \psi_\mathrm{con}\left({ \hat{d}\left(\bolds\left(\boldx\right)\right)}\right)\right]$
    \ENDWHILE
    \STATE \textbf{Output:} $\hat{\bolds}$.
\end{algorithmic}
\label{alg.search.s}
\end{algorithm}
\begin{theorem}\label{thm.dim.red}
    Suppose $h$ is associated with an $f$-divergence $D_\phi$ according to Theorem \ref{thm.gf} and $D_\psi$ is the mirror of $D_\phi$. 
    If $r(\boldx) = g^*(\bolds^*(\boldx))$, then $\bolds^*$ must be an $\arg \sup$ of the following objective:
    \begin{align}
    \label{eq.dim.red}
        \sup_{\bolds} \mathbbE_{p} \left[ h(r^\circ(\bolds(\boldx))) \right] - \mathbbE_{q} \left[ \psi_\mathrm{con}(h(r^\circ(\bolds(\boldx)))) \right]. 
    \end{align}
\end{theorem}
\begin{proof}
    Since $r(\boldx) = g^*(\bolds^*(\boldx))$, Proposition~\ref{def.hrx.estimator} implies that
    $(g^*, \bolds^*)$ is necessarily an $\arg\sup$ to the following optimization problem:
    \begin{align*}
        &\sup_{\bolds, g} \mathbbE_{p} \left[ h(g(\bolds(\boldx))) \right] - \mathbbE_{q} \left[ \psi_\mathrm{con}(h(g(\bolds(\boldx)))) \right].
    \end{align*}
    Due to the law of unconscious statistician, $\mathbbE_p[f(\bolds(\boldx))] = \mathbbE_{p^\circ}[f(\boldz)]$, where $\boldz = \bolds(\boldx)$. The above optimization problem can be rewritten as  
    \begin{align*}
        \sup_{\bolds} \sup_g \mathbbE_{p^\circ} \left[ h(g(\boldz)) \right] - \mathbbE_{q^\circ} \left[ \psi_\mathrm{con}(h(g(\boldz))) \right], 
    \end{align*} Proposition~\ref{def.hrx.estimator} states that for all $\bolds$, $g = r^\circ$ is an $\arg \sup$ of the inner optimization problem. Substituting this optimal solution of $g$ and rewriting the expectation using $\boldx$ again, we arrive
    \begin{align*}
        \sup_{\bolds} \mathbbE_{p} \left[ h(r^\circ(\bolds(\boldx))) \right] - \mathbbE_{q} \left[ \psi_\mathrm{con}(h(r^\circ(\bolds(\boldx)))) \right]. 
    \end{align*}
\end{proof}

Both expectations in \eqref{eq.dim.red} can be approximated using samples from $p$ and $q$. Given a fixed $\bolds$, $h( r^\circ(\boldz))$ can be approximated by an $m$-dimensional local linear interpolation 
\begin{align}
\label{eq.opt.wb}
    h(r^\circ(\boldz)) \approx \hat{d}(\boldz) :=
    \langle  \hat{\boldw}\left(\boldz\right), \boldz\rangle + \hat{b}\left(\boldz\right)
    , \left(\hat{\boldw}(\boldz), \hat{b}(\boldz)\right) := \argmin_{\boldw\in \mathbbR^{\color{red}m}, b\in \mathbbR} \ell(\boldw, b; \boldz, \mathcal{D}_{{\color{red}{p^\circ}}},\mathcal{D}_{\color{red}{q^\circ}}),
\end{align}
where $\mathcal{D}_{{\color{red}{p^\circ}}},\mathcal{D}_{\color{red}{q^\circ}}$ are sets of samples from $p^\circ$ and $q^\circ$ respectively. 

Approximating expectations in \eqref{eq.dim.red} with samples in $\mathcal{D}_p$ and $\mathcal{D}_q$ and replacing $h\circ r^\circ$ with $\hat{d}$, 
we solve the following optimization to obtain an estimate of $\bolds$:
\begin{align}
\label{eq.opt.s}
    \hat{\bolds} := \argmax_{\bolds\in \mathbb{S}} \frac{1}{n_p} \sum_{i=1}^{n_p}  \hat{d}\left[\bolds\left(\boldx_p^{(i)}\right)\right] 
    - \frac{1}{n_q} \sum_{i=1}^{n_q} \psi_\mathrm{con}\left[{ \hat{d}\left(\bolds(\boldx_q^{(i)})\right)}\right]. 
\end{align}
The optimization of 
\eqref{eq.opt.s} 
is a bi-level 
optimization problem as $\hat{d}$ depends on \eqref{eq.opt.wb}. We propose to divide the whole problem into two steps: First, let $\bolds = \hat{\bolds}$ and solve for $(\hat{\boldw},\hat{b})$.  Then, with the estimated $(\hat{\boldw},\hat{b})$, we solve for $\hat{\bolds}$. Repeat the above procedure until convergence. 
This algorithm is detailed in Algorithm \ref{alg.search.s}. 

In practice, we restrict $\mathbb{S}$ to be the set of all linear maps via a matrix $\boldS \in \mathbb{R}^{d \times m}$ whose columns are orthonormal basis, i.e., $\bolds(\boldx) := \boldS^\top \boldx$. The Jacobian $\nabla \bolds(\boldx)$ is simply $\boldS^\top$. 

After obtaining $\hat{\bolds}$, we can approximate the gradient flow using the chain rule described in \eqref{eq.dimred}: 
\[\nabla (h\circ r)(\boldx^\star) \approx \nabla^\top \hat{\bolds}(\boldx^\star) \hat{\boldw}(\boldz^\star),\]
where $\hat{\boldw}(\boldz^\star)$ is approximated by an $m$-dimensional local linear interpolation  
$$
\left(\hat{\boldw}(\boldz^\star), \hat{b}(\boldz^\star)\right) := \argmin_{\boldw\in \mathbbR^{\color{red}m}, b\in \mathbbR} \ell(\boldw, b; \boldz^\star, \mathcal{D}_{{\color{red}{p^\circ}}},\mathcal{D}_{\color{red}{q^\circ}})\mid_{\bolds = \hat{\bolds}}.
$$

\section{Discussion: Sufficient Condition of Density Ratio Preserving Map}
\label{sec.density.ratio.pres.proof}

In this Section, we provide a sufficient condition for $\bolds$ to be a density ratio preserving map.
\begin{lemma}
If there exists some $g:\mathbbR^m \to \mathbb{R}_+$, such that 
$r(\boldx) =  g(\bolds(\boldx)), $
then $\bolds$ is a density ratio preserving map and
$
g(\bolds(\boldx)) = r^{\circ} (\bolds(\boldx)).
$
\end{lemma}

\begin{proof}
    The statement $g(\bolds(\boldx))$ being a density ratio $\frac{p(\boldx)}{q(\boldx)}$ is equivalent to asserting that $\mathrm{KL}[p, q \cdot (g\circ \bolds)] = 0$.
    Since KL divergence is always non-negative, it means
    \begin{align}
    \label{eq.opt1}
         g = \argmin_{g: \int q(\boldx) 	 g(\bolds(\boldx))  \dx = 1, g \ge 0 } \mathbbE_p \left[ \log \frac{p(\boldx)}{q(\boldx)g(\bolds(\boldx))} \right] = \mathbbE_q [ \log g(\bolds(\boldx))] + C_1, 
    \end{align}    
    i.e., $g$ is a minimizer of $\mathrm{KL}[p, q \cdot (g \circ \bolds)]$, where $g \ge 0$ is constrained in a domain where $q \cdot (g \circ \bolds)$ is normalized to 1.

    Similarly, $r^\circ(\boldz)$ being a density ratio $\frac{p^\circ(\boldz)}{q^\circ(\boldz)}$ is the same as asserting that $\mathrm{KL}[p^\circ, q^\circ r^\circ] = 0$ and is equivalent to
    \begin{align}
    \label{eq.opt2}
         r^\circ = \argmin_{g: \int q^\circ(\boldz) g(\boldz) \mathrm{d} \boldz  = 1, g \ge 0} \mathbbE_{p^\circ} \left[ \log \frac{p^\circ(\boldz)}{q^\circ(\boldz)g(\boldz)} \right] = \mathbbE_{q^\circ}[ \log g(\boldz) ] + C_2,
    \end{align}
    i.e., $r^\circ$ is a minimizer of $\mathrm{KL}[p^\circ, q^\circ g ]$. 

    In fact, one can see that \eqref{eq.opt1} and \eqref{eq.opt2} are identical optimization problems due to the law of the unconscious statistician: 
    $\int q(\boldx) 	 g(\bolds(\boldx))  \dx = \mathbbE_q [ g(\bolds(\boldx)) ] = \mathbbE_{q^\circ}[  \left(g(\boldz) \right) ] = \int q^\circ(\boldz) g(\boldz) \mathrm{d} \boldz$
    and $\mathbbE_q [ \log g(\bolds(\boldx))] = \mathbbE_{q^\circ}[ \log g(\boldz) ]$, which means their solution sets are the same. Therefore, for any $g$ that minimizes \eqref{eq.opt1}, 
    it must also minimize \eqref{eq.opt2}. 
    Hence it satisfies the following equality 
    $
        r^\circ(\bolds(\boldx)) = g(\bolds(\boldx)) = r(\boldx)
    $, where the second equality is by our assumption.
\end{proof}

\section{Discussion: Stein Variational Gradient Descent}
\label{sec.app.svgd}
SVGD minimizes $\mathrm{KL}[q_{t+1}, p]$, where samples of $q_{t+1}$ is constructed using the following deterministic rule: 
\begin{align}
\label{eq.svgd.update.nai}
    \boldx_{t+1} = \boldx_{t} + \eta \boldu_t(\boldx_{t}). 
\end{align}
$\boldx_t$ are particles at iteration $t$, $\boldu_t \in \mathcal{H}^d$, a $d$-dimensional Reproducing Kernel Hilbert Space (RKHS) with a kernel function $l(\boldx, \boldx')$. 
\citet{LiuQ2016SVGD} shows
the optimal update $\boldu_t$ 
has a closed form:
\begin{align}
\label{eq.svgd.og}
    \boldu^{\mathrm{svgd}}_t := \mathbbE_{q_t}[l(\boldx_t, \cdot) \nabla \log p(\boldx_t) + \nabla l(\boldx_t, \cdot)].
\end{align} 
In practice, expectations  $\mathbbE_{q_t}[\cdot]$ can be approximated by $\widehat{\mathbbE}_{q_t}[\cdot]$, i.e., the sample average taken from the particles at time $t$.

\citet{Chewi2020svgd} links SVGD with $f$-divergence WGF: $\boldu_t^\mathrm{svgd}$ is \textit{the backward KL divergence WGF} under the coordinate-wise transform of an integral operator. 
Indeed, the $i$-th dimension of SVGD update 
    $u^{\mathrm{svgd}}_{t,i}$ can be expressed as 
\begin{align}
\label{eq.svgd.gf}
        u^{\mathrm{svgd}}_{t,i} := & \mathbbE_{q_t} \left[l(\boldx_t, \cdot) \partial_i \log p(\boldx_t) \right]  + \mathbbE_{q_t} \left[ \partial_i l(\boldx_t, \cdot) \right] \notag \\
    = &\mathbbE_{q_t} \left[l(\boldx_t, \cdot) \partial_i \log p(\boldx_t) \right]  - \mathbbE_{q_t} \left[ l(\boldx_t, \cdot ) \partial_i \log q_t(\boldx_t) \right] \notag \\
    = &\int q_t(\boldx_t) l(\boldx_t, \cdot) \partial_i \log \frac{p(\boldx_t)}{q_t(\boldx_t)} \mathrm{d} \boldx_t,
\end{align}
where the last line is an integral operator \citep{wainwright_2019} of the functional  $\partial_i \log \frac{p}{q_t}$, i.e., the $i$-th dimension of the backward KL divergence flow $\nabla \log r_t$. \footnote{Note that the second equality in \eqref{eq.svgd.gf} is due to the integration by parts, and only holds under conditions that $\lim_{\|\boldx\| \to \infty} q_t(\boldx) = 0$.}
Due to the reproducing property of RKHS, the SVGD update at some fixed point $\boldx^\star$ can be written as 
\[ \boldu^\mathrm{svgd}(\boldx^\star) = \langle \boldu^\mathrm{svgd}_t, l(\cdot, \boldx^\star) \rangle_{\mathcal{H}^d}   = \mathbbE_q[l(\boldx, \boldx^\star)\nabla \log r(\boldx)].\]

\section{Additional Experiments}

\subsection{Gradient Estimation}
\begin{figure}
    \centering
    \includegraphics[width=.45\textwidth]{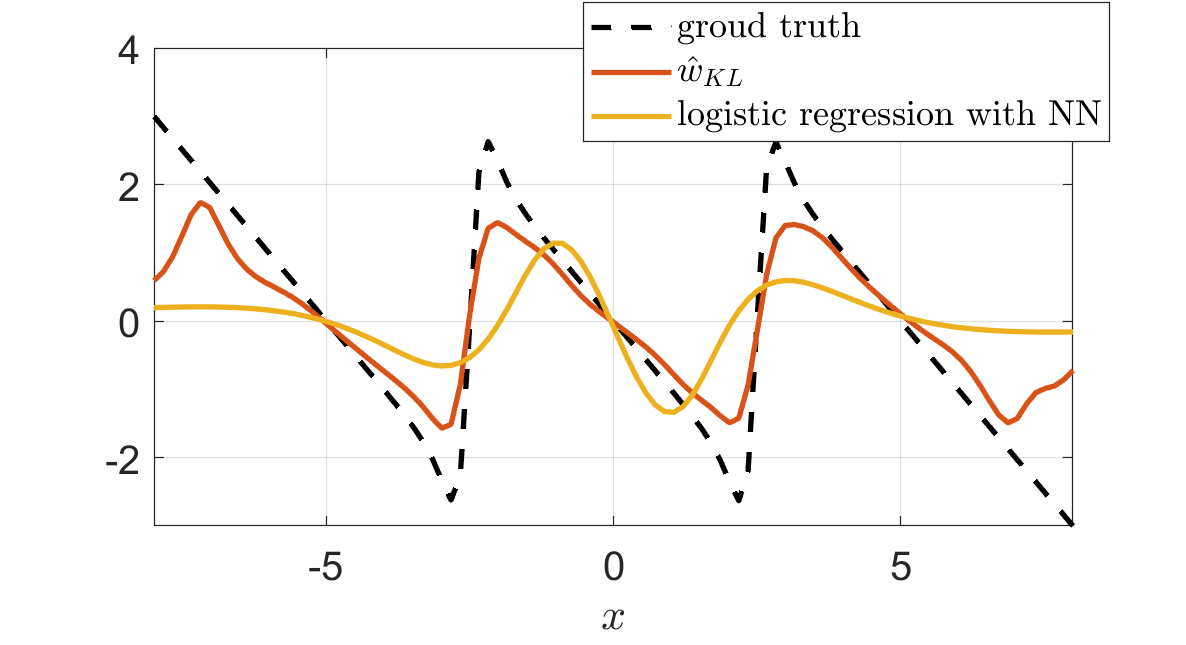}
    \includegraphics[width=.45\textwidth]{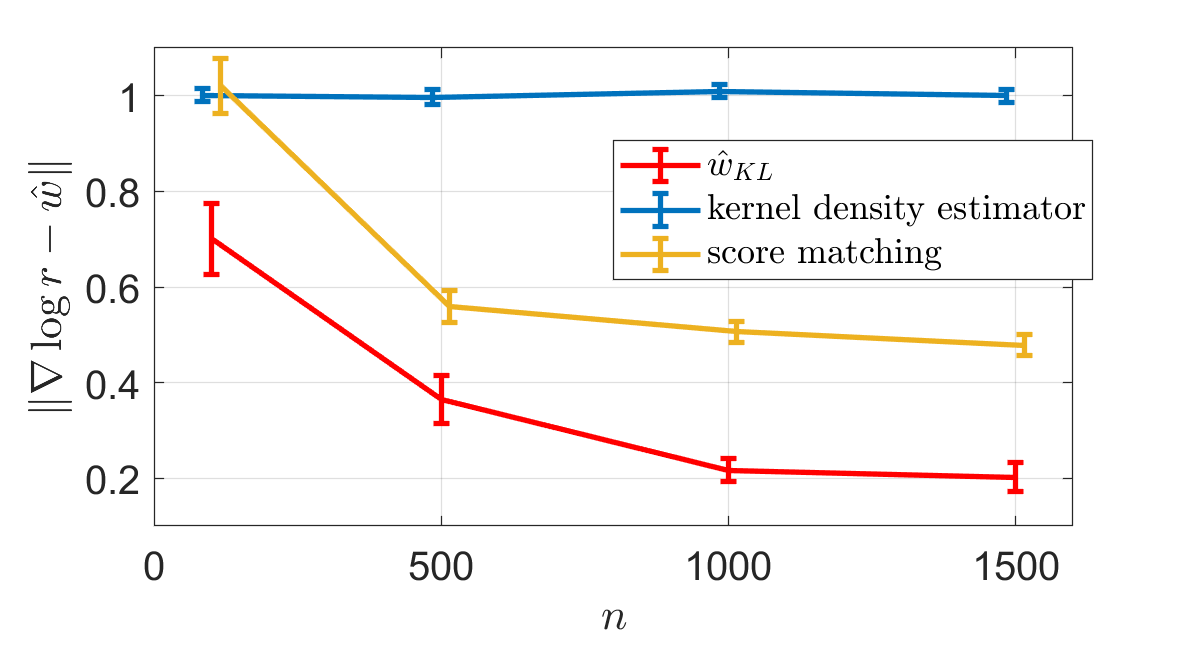}
    \caption{Left: $\nabla \log r(\boldx)$ and its approximations. Right:  Estimation error with standard error.}
    \label{fig.gradient.estimation}
\end{figure}
Now we investigate the performance of estimating $\nabla \log r(\boldx)$ using the proposed gradient estimator and an indirect estimator using logistic regression. For the indirect estimator,  we first train a Multilayer Perceptron (MLP) using a binary logistic regression to approximate $\log r$. 
Then obtain $\nabla \log r$ by auto-differentiating the estimated log ratio. The kernel bandwidth in our method is tuned by using the model selection criterion described in Section \ref{sec.model.selection}. 

To conduct the experiments, we let $p = \mathcal{N}(-5, .5)/3 + \mathcal{N}(0, .5)/3 + \mathcal{N}(5, .5)/3$  and $q = \mathcal{N}(-5, 1)/3 + \mathcal{N}(0, 1)/3 + \mathcal{N}(5, 1)/3$. From each distribution, 5000 samples are generated for approximating the gradients. 

The left plot in Figure \ref{fig.gradient.estimation} shows the true gradient and its approximations. It can be seen that the direct gradient estimation is more accurate than estimating the log ratio first then taking the gradient. 

The right plot in Figure \ref{fig.gradient.estimation}  displays the estimation errors of different methods,  comparing the proposed method with Kernel Density Estimation (KDE) and score matching, all applied to the same distributions in the previous experiment.  
KDE was previously used in approximating WGF \citep{wang2022projected}. It first estimates $p$ and $q$ with $\hat{p}$ and $\hat{q}$ separately using non-parametric kernel density estimators, then approximates $\nabla \log r$ with $\nabla \log \hat{p} - \nabla \log \hat{q}$ . 
The score matching approximates $\nabla \log \hat{p}$ and $\nabla \log \hat{q}$ with the minimizers of Fisher-divergence. It has also been used in simulating particle ODEs in a previous work \citep{maoutsa2020interacting}. 
The estimation error plot shows that the proposed estimator yields more accurate results compared to the other two kernel-based gradient estimation methods, namely KDE and score matching.  

\subsection{Wasserstein Gradient Flow for Generative Modelling}
\label{sec.wgf}
\begin{figure}
    \centering
    \includegraphics[width=.8\textwidth]{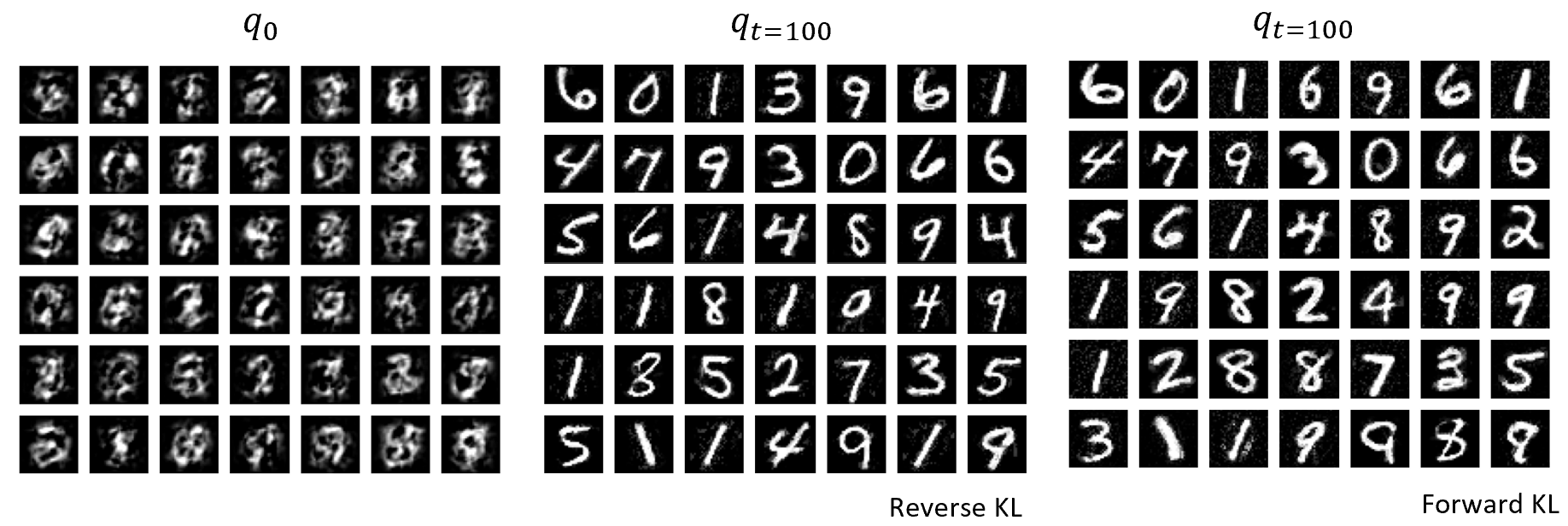}
    \caption{Samples generated using forward and backward KL velocity field.}
    \label{fig.mnist}
\end{figure}
In this experiment, we test the performance of the proposed gradient estimators by generating samples from a high dimensional target distribution (MNIST handwritten digits). 
We will check whether the quality of the particles can be improved by performing WGC using our estimated updates. Note that we do not intend to compare the generated samples with NN-based approaches, as the focus of our paper is on local estimation using kernel functions. 
We perform two different WGFs, forward KL and backward KL whose updates are approximated using $\hat{\boldw}_\rightarrow$ and $\hat{\boldw}_\leftarrow$ respectively. We let the initial particle distribution $q_0$ be $\mathcal{N}(\boldmu, \boldSigma)$, where $\boldmu, \boldSigma$ are the mean and covariance of the target dataset and fix the kernel bandwidth using ``the median trick'' (see the appendix for details).

The generated samples together with samples from the initial distribution $q_0$ are shown in Figure \ref{fig.mnist}.
Judging from the generated sample quality, it can be seen that both VGDs perform well and both have made significant improvements from the initial samples drawn from $q_0$. 

We also provide two videos showing the animation of the first 100 gradient steps:
\begin{itemize}
    \item Forward KL: \url{https://youtube.com/shorts/HZcvUykrpbc}
    \item Backward KL: \url{https://youtube.com/shorts/AgN6dsDecCM}
\end{itemize}

\subsection{Feature Space Wasserstein Gradient Flow}
\label{sec.exp.fea}
\begin{figure}
    \centering
    \includegraphics[width = .8\textwidth]{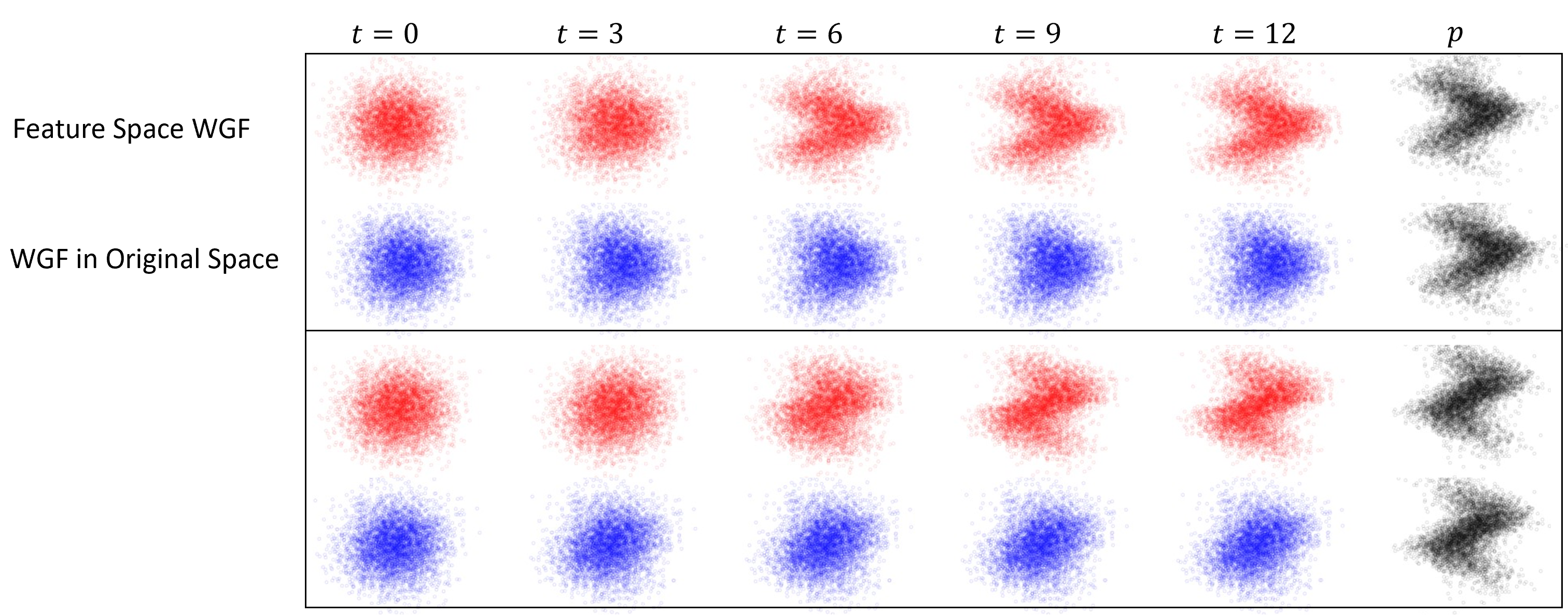}
    \caption{$\{\boldx_{q_t}\}$ in the first two dimensions $x_1, x_2$. 
    Feature space WGF converges in fewer steps compared to WGF in the original space. Above: $g = \cos$. Below: $g = \sin$. }
    \label{fig.par.fea}
\end{figure}

In this experiment, we run WGF in a 5-dimensional space. The target distribution is \[p = p_{12}(\boldx_{1,2})p_{345}(\boldx_{3,4,5}), p_{345}:=\mathcal{N}(\boldzero, \boldI)\] 
and $p_{12}$ is constructed by a generative process. We generate samples in the first two dimensions as follows 
\[
    X_{1} = g(X_{2}) + \epsilon, X_{2} \sim \mathcal{N}(0, 1), \epsilon \sim \mathcal{N}(0, 1).
\]
In our experiment, we draw 5000 samples from $p$, and 5000 samples from $q_0 := \mathcal{N}(\boldzero, \boldI)$, and run the backward KL gradient flow. 
Clearly, this WGF has a low dimensional structure since $p$ and $q$ only differs in the first two dimensions. 
We also run feature space backward KL field whose updates are calculated using \eqref{eq.dimred}. The feature function $\bolds(\boldx) = \boldS^\top\boldx$ is learned by Algorithm \ref{alg.search.s}. 

The resulting particle evolution for both processes is plotted in Figure \ref{fig.par.fea}. For visualization purposes, we only plot the first two dimensions. It can be seen that the particles converge much faster when we explicitly exploit the subspace structure using the feature space WGF. In comparison, running WGF in the original space converges much slower. 

In this experiments, we set the learning rates for both WGF and feature space WGF to be 0.1 and the kernel bandwidth $\sigma$ in our local estimators is tuned using cross validation with a candidate set ranging from 0.1 to 2. 

\subsection{NW, LL and SVGD with Different Sample Sizes}
\label{sec.svgd.diff.n}
See Figure \ref{fig:svgd-diff-n}. 
\begin{figure}
    \centering
    \includegraphics[width=.98\textwidth]{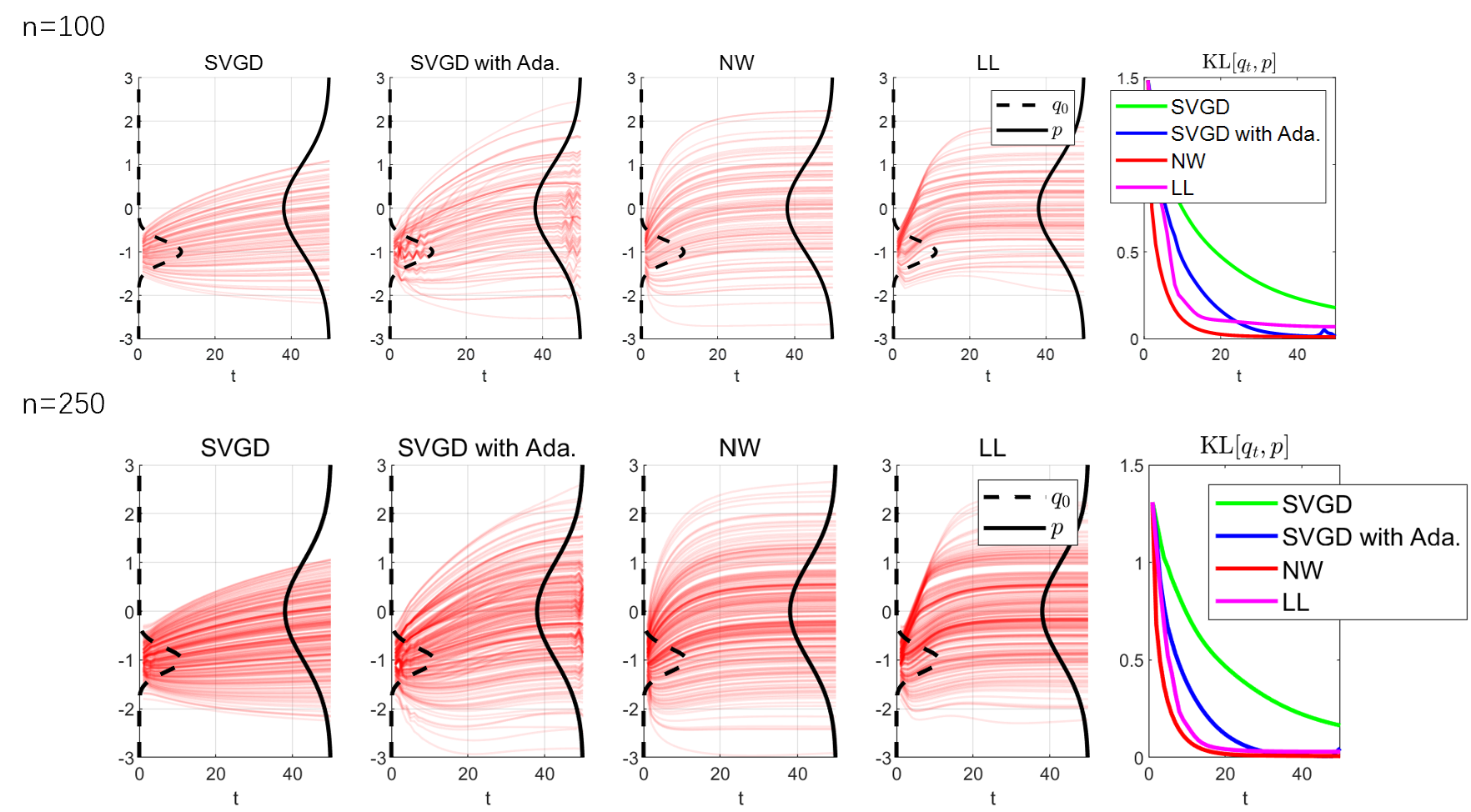}
    \caption{NW, LL, SVGD compared with sample size $n = 100, 250.$}
    \label{fig:svgd-diff-n}
\end{figure}

\end{document}